\documentclass[sigconf]{acmart}
\usepackage{multirow} 
\usepackage[linesnumbered, ruled]{algorithm2e}

\AtBeginDocument{%
  \providecommand\BibTeX{{%
    \normalfont B\kern-0.5em{\scshape i\kern-0.25em b}\kern-0.8em\TeX}}}

\copyrightyear{2023} 
\acmYear{2023} 
\setcopyright{acmlicensed}\acmConference[WWW '23]{Proceedings of the ACM Web Conference 2023}{May 1--5, 2023}{Austin, TX, USA}
\acmBooktitle{Proceedings of the ACM Web Conference 2023 (WWW '23), May 1--5, 2023, Austin, TX, USA}
\acmPrice{15.00}
\acmDOI{10.1145/3543507.3583225}
\acmISBN{978-1-4503-9416-1/23/04}

\begin{document}

\title{Web Photo Source Identification based on Neural Enhanced Camera Fingerprint}

\author{Feng Qian}
\authornote{All authors contributed equally to this research.}
\authornote{Corresponding Author.}
\email{youzhi.qf@antgroup.com}
\orcid{0000-0003-2922-9346}
\affiliation{%
  \institution{Ant Group}
  \country{China}
}

\author{Sifeng He}
\authornotemark[1]
\authornotemark[2]
\email{hsf215kg@gmail.com}
\affiliation{%
  \institution{Ant Group}
  \country{China}
}

\author{Honghao Huang}
\authornotemark[1]
\email{huanghonghao.hhh@antgroup.com}
\affiliation{%
  \institution{Ant Group}
  \country{China}
}

\author{Huanyu Ma}
\authornotemark[1]
\email{huanyu.mhy@antgroup.com}
\affiliation{%
  \institution{Ant Group}
  \country{China}
}

\author{Xiaobo Zhang}
\email{ayou.zxb@antgroup.com}
\affiliation{%
  \institution{Ant Group}
  \country{China}
}

\author{Lei Yang}
\email{yl149505@antgroup.com}
\affiliation{%
  \institution{Ant Group}
  \country{China}
}

\renewcommand{\shortauthors}{Qian and He, et al.}

\begin{abstract}
With the growing popularity of smartphone photography in recent years, web photos play an increasingly important role in all walks of life. Source camera identification of web photos aims to establish a reliable linkage from the captured images to their source cameras, and has a broad range of applications, such as image copyright protection, user authentication, investigated evidence verification, etc. This paper presents an innovative and practical source identification framework that employs neural-network enhanced sensor pattern noise to trace back web photos efficiently while ensuring security. Our proposed framework consists of three main stages: initial device fingerprint registration, fingerprint extraction and cryptographic connection establishment while taking photos, and connection verification between photos and source devices. By incorporating metric learning and frequency consistency into the deep network design, our proposed fingerprint extraction algorithm achieves state-of-the-art performance on modern smartphone photos for reliable source identification. Meanwhile, we also propose several optimization sub-modules to prevent fingerprint leakage and improve accuracy and efficiency. Finally for practical system design, two cryptographic schemes are introduced to reliably identify the correlation between registered fingerprint and verified photo fingerprint, i.e. fuzzy extractor and zero-knowledge proof (ZKP). The codes for fingerprint extraction network and benchmark dataset with modern smartphone cameras photos are all publicly available at https://github.com/PhotoNecf/PhotoNecf \footnote{The codes and data are also available with artifact DOIs: 10.5281/zenodo.7627667, 10.5281/zenodo.7627409 and 10.5281/zenodo.7627647.}. 
\end{abstract}




\maketitle

\section{Introduction}
With the explosive growth of social sharing platforms like Instagram, Twitter, TikTok, etc., a massive amount of web photos and video are generated. Consequently, source camera identification that acquires the photographing device information of an arbitrary web image can be availably utilized on wide range of applications. As shown in scenarios from Figure 1, source identification can effectively determine the identity of original creator of the web photos to prevent piracy and protect copyright \cite{DBLP:journals/jthtl/Stewart12,DBLP:journals/corr/cs-CY-0311054,DBLP:journals/tnse/WangSWZ22}. By verifying the photo traceability \cite{DBLP:conf/blockchain2/IgarashiKKKD21}, this framework can also evaluate the trustworthiness of the sensed data. In the case of investigation and evidence collection, source identification can also help to determine the photography device, thereby assisting in product traceability \cite{DBLP:journals/corr/abs-2207-01323,DBLP:journals/sensors/YaoZ22,DBLP:journals/tii/CaoJM20} and media forensics \cite{DBLP:journals/jimaging/FerreiraAC21,DBLP:phd/ethos/Quan20,DBLP:journals/data/FerreiraAC21} cases. In addition, source identified devices of collected images can also be served as an important factor for user authentication system \cite{DBLP:journals/tifs/GaleaF18,DBLP:conf/uss/XuPFM16,DBLP:journals/amm/LiuZHL18}.

\begin{figure}[t!]
\centering
\includegraphics[width=\linewidth]{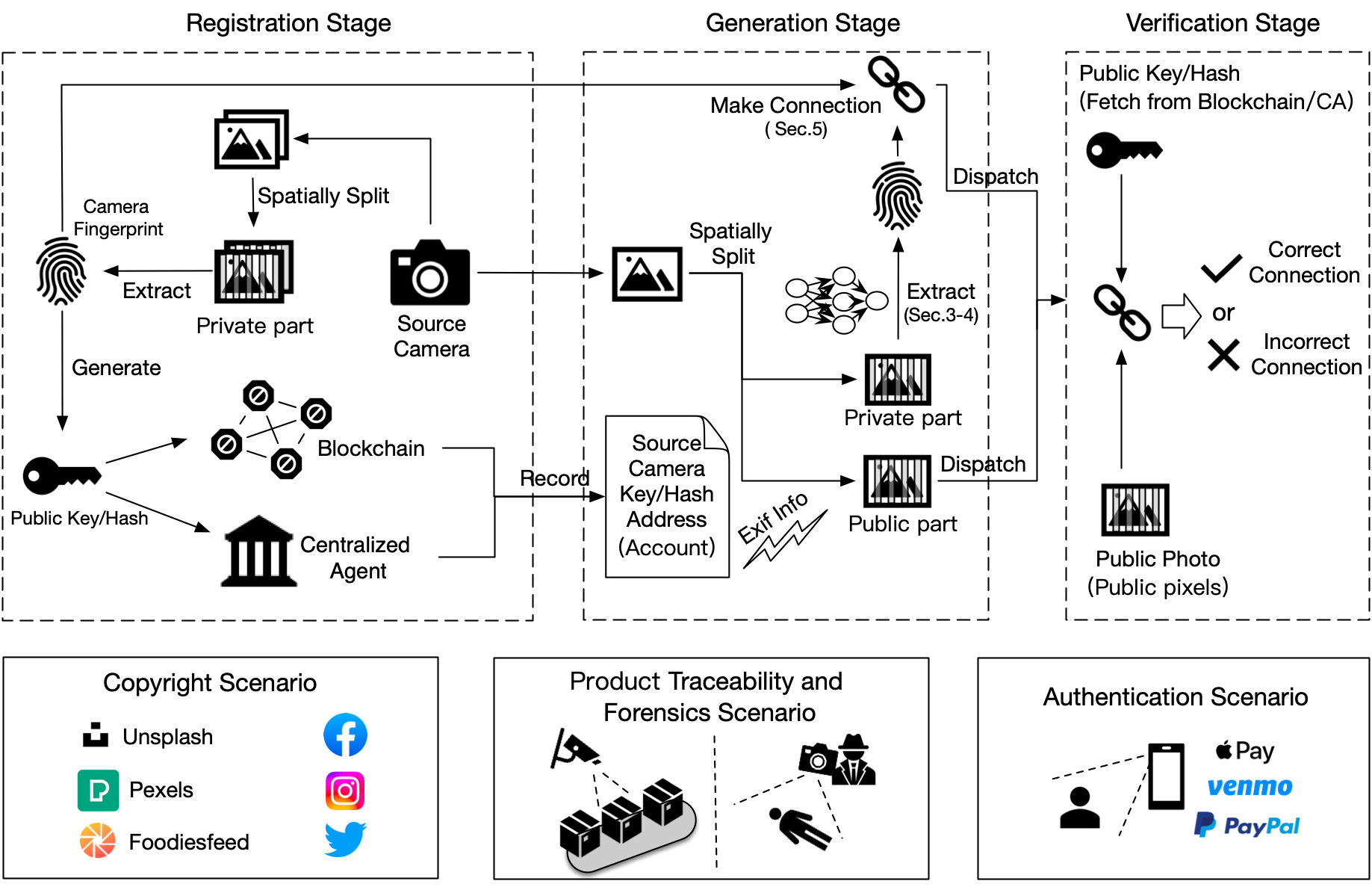} 
\caption{Three main stages of source identification framework mentioned in the Abstract, and several example application scenarios of source identification systems.}
\label{fig1}
\end{figure}

In order to achieve this, previous works utilize sensor pattern noise as the corner stone of digital image forensics \cite{lukas2006digital,chen2008determining}. This sensor pattern noise, also known as camera fingerprint, is presented in each photo and is only associated with the source device rather than high-level semantic content of the photo \cite{lawgaly2016sensor}. 
In detail, digital images can be traced back to their sensors based on unique noise characteristics. Minute manufacturing imperfections are believed to make each sensor physically unique, leading to the presence of a weak yet deterministic sensor pattern noise (SPN) in each photo captured by the same sensor \cite{fridrich2013sensor}. This fingerprint, previously referred to as photo-response non-uniformity (PRNU), can be estimated from images captured by a specific camera for purpose of source camera identification, in which a noise signal extracted from a probe image of unknown provenance is compared against pre-computed fingerprint estimates from candidate cameras \cite{quiring2019security}.

The PRNU is often estimated in the form of the noise residual of an image. The noise residual can be extracted from an image by simply subtracting the denoised image from the original image. Most previous methods obtain denoised image by applying some high-pass filters in the spatial or transform (Fourier, DCT, wavelet) domain \cite{farid2003higher, gou2007noise, he2012digital, li2016identification}. In the conventional PRNU extraction algorithm \cite{lukas2006digital}, the denoising filter adopts the wavelet-based denoising filter which will be introduced in Section 2.2. Noise residuals can be also used in a blind context (no external training) to reveal local anomalies that indicate possible image tampering \cite{lyu2014exposing, cozzolino2015splicebuster}.

In practice, we still have to solve the following challenges based on PRNU: 
1. With the widespread use of smartphones and gradual development of image signal processor (ISP) \cite{Yoshimura2022DynamicISPDC}, will the performance of conventional PRNU algorithm still be guaranteed?
2. In the practical application of PRNU, there are still problems of fingerprint leakage, low accuracy and efficiency \cite{gloe2012unexpected}. 
3. How to effectively apply PRNU algorithm to real scenarios of source identification while ensuring high reliability and security?

To address the above issues, we propose the following solutions:

First, we propose a novel camera fingerprint extraction algorithm based on denoising neural network. In contrast to previous network design with only supervision to approximate the pre-computed PRNU fingerprint \cite{kirchner2019spn}, we further leverage a Deep Metric Learning (DML) framework based on a triplet-wise scheme, which has been shown to be effective in a variety of cases \cite{wang2014learning,kordopatis2017near, ge2018deep}. Meanwhile, we also supplement an additional frequency loss \cite{jiang2021focal} to realize frequency consistency between the predicted noise residual and pre-computed PRNU fingerprint, thereby further improving the stability of fingerprint extraction. Finally, considering the existence of color filter array (CFA) in the imaging process, we introduce Pixel Shuffle operation \cite{shi2016real} into our network. Based on our proposed neural enhanced algorithm, camera fingerprint can be accurately extracted from limited number of RAW photos and it shows significantly higher performance than PRNU results.

Second, we directly extract camera fingerprint from RAW images rather than other compressed formats such as JPEG. Meanwhile, only the splitted part of photo (e.g. only the even lines of pixels) are made public, and the remainder pixels of the photo are privately utilized for fingerprint extraction and comparipliton. Therefore, the camera fingerprint which is extracted from private part can not only be survived from ISP process, but also cannot be leaked for adversarial attacks. Meanwhile, under the theoretical guidance of Cramer–Rao lower bound (CRLB) on the fingerprint variance \cite{chen2008determining}, we further propose two optimization sub-modules (block filtering and burst integration) to improve fingerprint accuracy which can also be broadly applied to different fingerprint extraction algorithms. Besides, we also utilize binary quantization of fingerprints \cite{bayram2012efficient} to improve computational efficiency.

Lastly, we design two novel source identification systems that rely on camera fingerprint extraction algorithm and cryptography schemes. In the first design, the camera fingerprint is compressed to compose a stable private key, and the detailed implementation is similar to PRNU application on user authentication \cite{valsesia2017user}. According to the verification of signature information based on compressed fingerprint, the source device of photo can be easily obtained. The second scheme is to combine zero-knowledge proof (ZKP) \cite{goldreich1994definitions} with camera fingerprint. ZKP protocol (e.g., zkSNARKs \cite{groth2016size, ben2014succinct}) formulates the complete processes (e.g., noise extraction, fingerprint matching, digest generation and matching, etc.) into circuit, creates proof and verifies the proof, therefore achieving traceability and verification of the whole process without data and privacy leakage. With saving source camera's public key/hash address as a meta data, the generated signature/proof flows with web image on the Internet, the image source identification can be verified at any time.

The contributions of this work can be summarized as follows:

\begin{itemize}
\item We have made a significant progress on the conventional camera fingerprint algorithm, reducing the identification error rate of models from 40.62\% to 2.345\%.
\item In order to ensure the privacy and performance of the camera fingerprint, we propose several additional beneficial sub-modules and prove their validity for error rate < 0.5\%.
\item We release a new dataset for benchmark with photos taken from recently announced smartphones. This dataset contains 1,665 photos taken from 15 iPhone cameras and 1,276 photos taken from 15 Android cameras, both in RAW format. 
\item We incorporate cryptography schemes into the overall framework design to improve the stability and security of the system, and complete their project implementation.
\end{itemize}

\section{Background}
\subsection{Sensor noise fingerprints}
Due to sensor element manufacturing imperfections, each camera photo does not only contain the original noise-free image content $I^0$, but also the sensor pattern noise $K$ as a camera-specific, multiplicative noise factor. A common simplified model of the image capturing process assumes the final image $I$ to take the form \cite{fridrich2013sensor}
\begin{equation}
\begin{aligned}
  I = I^0 + I^0K + \Gamma 
  \label{eq1}
\end{aligned}
\end{equation}
where $\Gamma$ reflects a variety of other additive noise terms. Due to its multiplicative nature, the pattern noise is not present in images with dark scene contents (i. e., $I^0\approx 0$). Extensive experiments have demonstrated that the noise factor $K$ represents a unique and robust camera fingerprint \cite{goljan2009large} that can be estimated from a number of images $I_1,...,I_N$ taken with a given camera of interest. The standard approach utilizes a denoising filter $F(\cdot )$ and models noise residuals $W_k = I_k - F(I_k)$ as in Fridrich's work \cite{fridrich2013sensor}:
\begin{equation}
\begin{aligned}
  W_k = I_kK + \Theta _k 
  \label{eq2}
\end{aligned}
\end{equation}

Modeling noise $\Theta$ subsumes $\Gamma$ and residues of the image content due to inherent imperfections of the denoising filter in separating image content from noise. Adopting an independent and identically distributed (i.i.d.) Gaussian noise assumption for $\Theta$, the maximum likelihood estimator of $K$ is \cite{fridrich2013sensor}
\begin{equation}
\begin{aligned}
  \hat{K} =  \frac{\sum_{k=1}^{N}W_kI_k}{\sum_{k=1}^{N}(I_k)^2} 
  \label{eq3}
\end{aligned}
\end{equation}

Given a query image $J$ of unknown provenance, camera identification then works by computing the residual $W_J = J - F(J)$,and evaluating its similarity to a camera fingerprint estimate against a set threshold $\tau$,
\begin{equation}
\begin{aligned}
  \phi_{W_J,\hat{K} } = \mathrm{sim} (W_J,\hat{K})> \tau  
  \label{eq4}
\end{aligned}
\end{equation}

Suitable similarity measures for this task are normalized cross-correlation or peak-to-correlation energy \cite{fridrich2013sensor, lukas2006digital}.

\subsection{Conventional PRNU extraction}
According to Section 2.1, the main algorithm process to obtain the noise pattern is the denoising filter. In the conventional PRNU extraction proposed by Lukas et al. \cite{lukas2006digital}, it is constructed in the wavelet domain. Image default size is a grayscale 512$\times$512 image. Larger images can be processed by multiple blocks and color images are denoised for each color channel separately. The high-frequency wavelet coefficients of the noisy image are modeled as an additive mixture of a locally stationary i.i.d. signal with zero mean (the noise-free image) and a stationary white Gaussian noise $N(0, \sigma _{0}^{2} )$ (the noise component). The denoising filter is built in two stages. In the first stage, the local image variance is estimated, while in the second stage the local Wiener filter is used to obtain an estimation of the denoised image in the wavelet domain. $\sigma _{0}$ is set to 5 (for dynamic range of images 0, ..., 255) to be conservative and to make sure that the filter extracts substantial part of the PRNU noise even for cameras with a large noise component. The detailed implementation can be inferred in the work by Lukas et al. \cite{lukas2006digital}.

\subsection{Limitation of current fingerprint }
PRNU has been proven effectively on cameras in the early years \cite{goljan2009large}. However, as the popularity of smartphones embedded with computational photography process, the effectiveness of PRNU algorithm needs to be further verified or improved. Meanwhile, as mentioned earlier, PRNU algorithm requires a registration process of $N$ images, which is unrealistic in many scenarios. Therefore, the accuracy performance and operational feasibility are the main challenges for applications of camera fingerprint. 

On the other hand, the system security also needs to be considered against fingerprint copy and abusing attack. In this case, the objective of the adversary is to impersonate a legitimate user and authorize a malicious request \cite{ba2018abc}. We also assume that the adversary can access the public photos that the victim captures with her smartphone. Those images may be hard to be kept private anyway, for example, pictures shared through online social networks such as Wechat or Facebook. Therefore, an adversary could estimate the victim smartphone’s fingerprint from public images and embed the obtained fingerprint into an image captured by her own device. Hence, the security of camera fingerprint algorithm in practical scenarios becomes another critical concern.

\section{Fingerprint extraction network}


\begin{figure*}
\begin{minipage}[c]{0.7\linewidth}
\includegraphics[width=\linewidth]{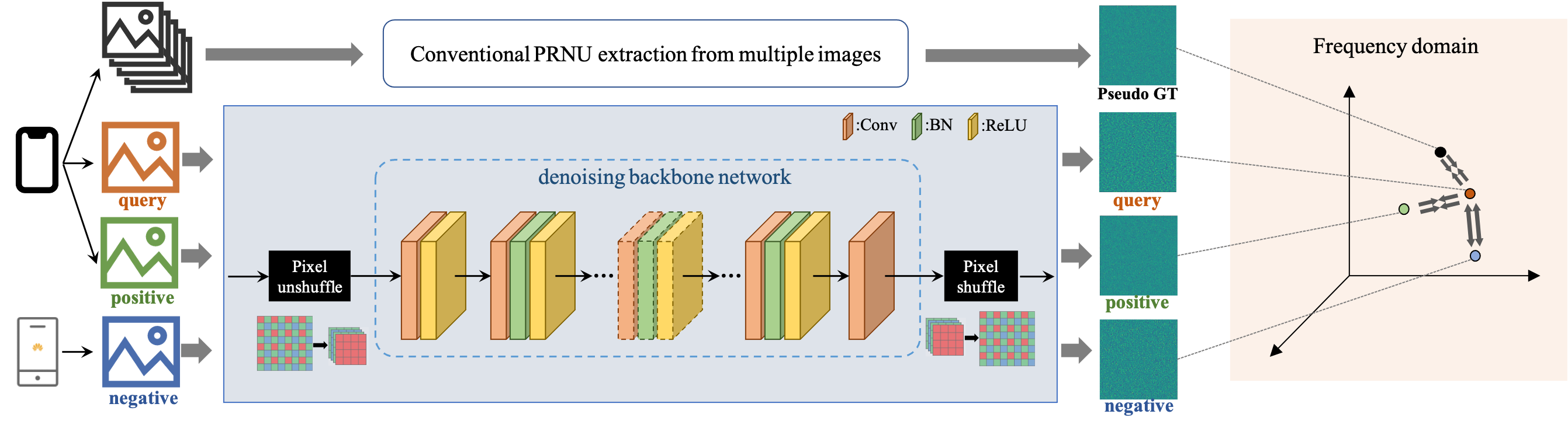}
\caption{Overview of fingerprint extraction network.}
\end{minipage}
\hfill
\begin{minipage}[c]{0.29\linewidth}
\includegraphics[width=\linewidth]{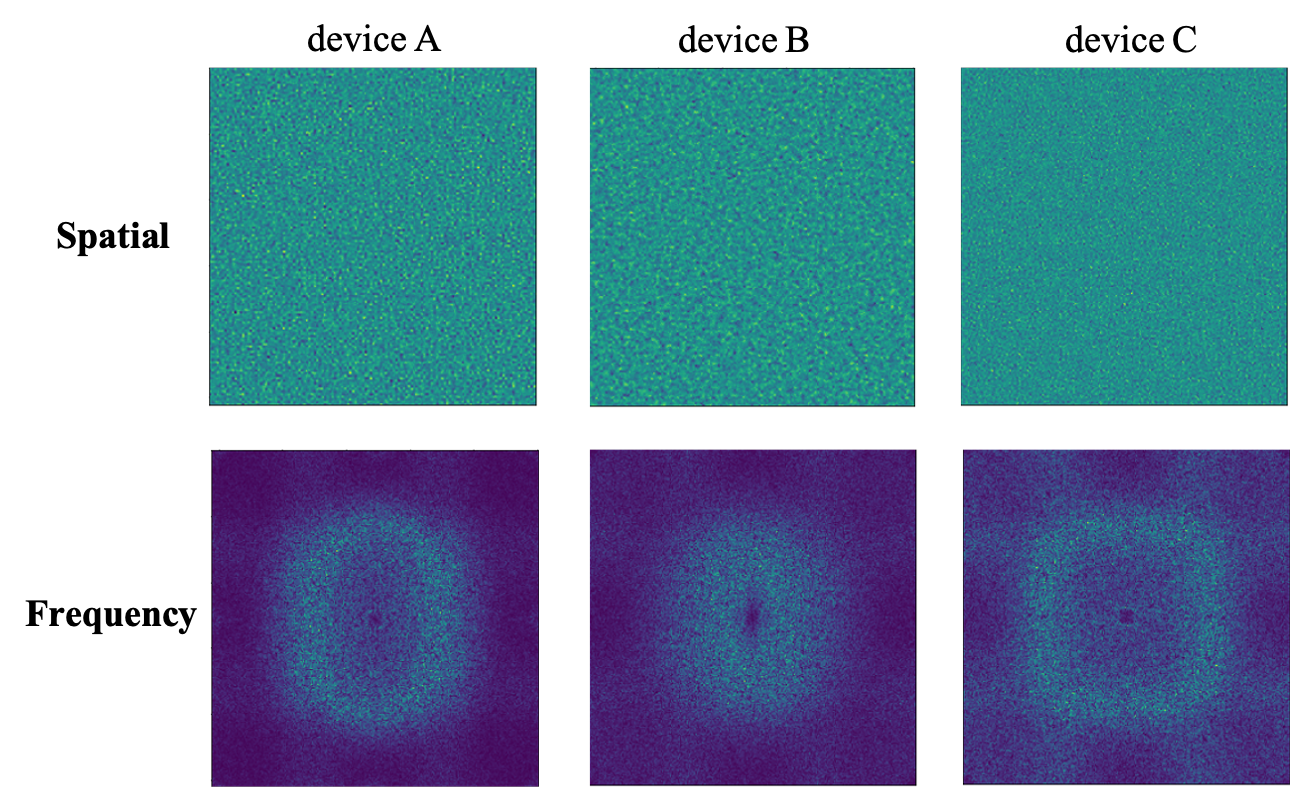}
\caption{Three example fingerprints in spatial and frequency domain.}
\end{minipage}%
\end{figure*}

As mentioned in Section 2.2, the main component of fingerprint extraction is the denoising part to obtain the noise residual $W_k$ of the image. Our goal is to improve the noise residual extraction process, thereby enhancing the individual device sensor artifacts for better identification results. Therefore, the algorithm takes a generic image as input and produces a suitable noise residual as output for next-step fingerprint matching. In this section, we describe our proposed fingerprint extraction network with unique loss and training design. The network overview is indicated in Figure 2.

For the target of obtaining the accurate noise residual for high confidence matching, the main challenge is the difficulty to obtain the ground truth (GT) pattern noise signals, which theoretically requires accurate instruments to measure \cite{wei2020physics}. In order to avoid the hardware-based measurement cost, we propose two methods to address this issue, i.e., deep metric learning and frequency correspondence with pseudo GT.  

Inspired from deep metric learning with successful applications on image embedding, we create a collection of training instances organized in the forms of triplets. Each triplet contains a query image $I_q$, a positive image $I_p$ (a photo from the same camera as the query) and a negative image $I_n$ (a photo from a different camera as the query). For the correlation distance on the embedding space after noise extraction network, the loss of a triplet $(I_q, I_p, I_n)$ is :
\begin{equation}
\begin{aligned}
  L_1 = \mathrm{max} (0, \mathrm{d}[\mathbf{DN} (I_q), \mathbf{DN} (I_p)] - \mathrm{d}[\mathbf{DN} (I_q), \mathbf{DN} (I_n)] + \gamma)  
  \label{eq5}
\end{aligned}
\end{equation}
where $\mathbf{DN}$ is a image denoising backbone network with residual noise as output, $\gamma$ is a margin parameter to ensure a sufficiently large difference between the positive-query distance and negative-query distance, and $\mathrm{d[\cdot , \cdot]}$ is the similar distance between noise residuals which can be measured with Euclidean distance or cosine similarity. We minimize this loss, which pushes the noise embedding distance from same camera $\mathrm{d}[\mathbf{DN} (I_q), \mathbf{DN} (I_p)]$ to 0 and $\mathrm{d}[\mathbf{DN} (I_q), \mathbf{DN} (I_n)]$ to be greater than $\mathrm{d}[\mathbf{DN} (I_q), \mathbf{DN} (I_p)] + \gamma$. In addition, we can also utilize batch hard strategy \cite{hermans2017defense} to search hardest positive and hardest negative within a batch of image dataset for each query sample to yield better performance. With an appropriate triplet generation strategy in place, the model will eventually learn a noise representation (fingerprint) that improves source identification performance.

Another optimization approach for overcoming the difficulty of obtaining ground truth sensor pattern noise signal is to approximate the fingerprint using conventional PRNU algorithm, which can guide and optimize the network at the early training stage. We refer to the fingerprint extracted by the PRNU algorithm as pseudo GT. Here, we can obtain a more accurate approximation of fingerprint by multiple photos as Eq.(3), and this calculation can be processed offline so as not to slow down the training time. An import observation is that most of camera fingerprints are not visually distinguishable from each other as can be shown in Figure 3, but in frequency domain they are apparently different to easily tell apart. Inspired from focal frequency loss for image reconstruction and synthesis in the work by Jiang et al. \cite{jiang2021focal}, we utilize the frequency distance between predicted residual noise of image $I_q$ and estimated PRNU $\hat{K}$ as the second part of loss:
\begin{equation}
\begin{aligned}
  L_2 =  \mathrm{d}[\mathfrak{F}(\mathbf{DN} (I_q)), \mathfrak{F}(\hat{K})]
  \label{eq6}
\end{aligned}
\end{equation}
where $\mathfrak{F}$ denotes 2D discrete Fourier transform, and here we directly use Euclidean distance in $\mathrm{d[\cdot , \cdot]}$ as consistent with frequency loss implementation in \cite{jiang2021focal}. $\hat{K}$ is the fingerprint calculated as Eq.(3) of the same camera as image $I_q$.


Therefore, the final loss of our proposed network is $L_1 + L_2$. Besides the unique loss design, we also introduce some other efficient operations to further improve the fingerprint extraction accuracy. Considering the Bayer filter mosaic of color filter array (CFA) of the camera sensor, one of the inductive bias of convolution layer, i.e., translation invariance, can actually affect the network performance. To solve this problem, we introduce the Pixel Shuffle operation \cite{shi2016real} which is commonly used in super-resolution scenarios into our network, as shown in Figure 2. We first implement sub-pixel convolutions with a stride of 2 for downsampling, therefore obtaining multiple channels and the same color filter of each channel. At last step of the network, we utilize efficient sub-pixel convolution with a stride of 1/2 and obtain noise residual with the same image size as input.The entire network design is also indicated in Figure 2.

\section{Fingerprint optimization modules}
In this section, we further optimize fingerprint extraction in terms of security, algorithm effectiveness and efficiency. 


\noindent\textbf{Leakage prevention.} As mentioned before, an adversary could estimate the victim smartphone’s fingerprint from public images and embed the obtained fingerprint into an image captured by her own device. Therefore, the main challenge is not to reveal fingerprints (directly or indirectly) while maintaining the dominated information of the captured images. One solution is to extract the fingerprint from RAW image and obtain its residual noise which is dominated by high-frequency components. Since web photos are usually processed with JPEG compression as a low-pass filter, the high-frequency components are unavailable or severely degraded in publicly available images, and fingerprint can only be estimated if one has access to the RAW data. In addition, we also spatially split the original images into two parts that are adjacent to each other on pixels, as shown in Figure 4(a). Only part of original image (i.e., even rows) is opened to public and remainder part (i.e., odd rows) is privately used for fingerprint calculation and comparison. Hence, based on the public even-part of image, the original photo can be easily obtained by upsampling, and the adversary cannot derive the fingerprint in private (odd) part from the public web photos. Further on in this paper, we refer the odd rows of photo as RAW odd photo, and the even rows of photo as RAW even photo.

\noindent\textbf{Accuracy improvement.} The estimated camera fingerprint $\hat{K}$ can be derived from Eq.(3). By computing the Cramer–Rao lower bound (CRLB) \cite{chen2008determining} on the variance of $\hat{K}$
\begin{equation}
\begin{aligned}
  var(\hat{K} ) \ge \frac{1}{-E[\frac{\partial ^2 L(K)}{\partial K^2} ]} = \frac{\sigma ^2}{\textstyle \sum_{k=1}^{N}(I_k)^2 } 
  \label{eq7}
\end{aligned}
\end{equation}

Eq.(7) informs us what images are best for the estimation of $\hat{K}$. The luminance (pixel value) of image should be as high as possible but not saturated, and larger $N$ is preferred for higher lower bound of $\hat{K}$. Based on these two factors, we propose block filtering and burst integration to further enhance the fingerprint. As for block filtering, we split the original image into multiple small blocks (e.g., 64 $\times$ 64) and obtain the individual weight of each block according to their average pixel luminance. The weight mask based on luminance can be float values or binary scores on selected threshold or fixed percentage, as shown in Figure 4(b). Then the similarity measure is the weighted sum of normalized correlation of each block. The detailed parameters and experiments are inferred in Section 6. Secondly, burst integration simply estimates fingerprints from continuously taken $N$ images instead of only one image, which can suppress other random noises such as scatter noise, readout noise. As shown in Figure 4(c), we also utilize maximum likelihood estimator (MLE) similar to Eq.(3) to obtain the optimized fingerprint from multiple burst photos. Nowadays, burst photography has become an important technology in computational imaging inside ISP and this can be easily realize by bottom layer API \cite{DBLP:journals/corr/abs-2102-09000}.

\begin{figure}[htb]
\centering
\includegraphics[width=\linewidth]{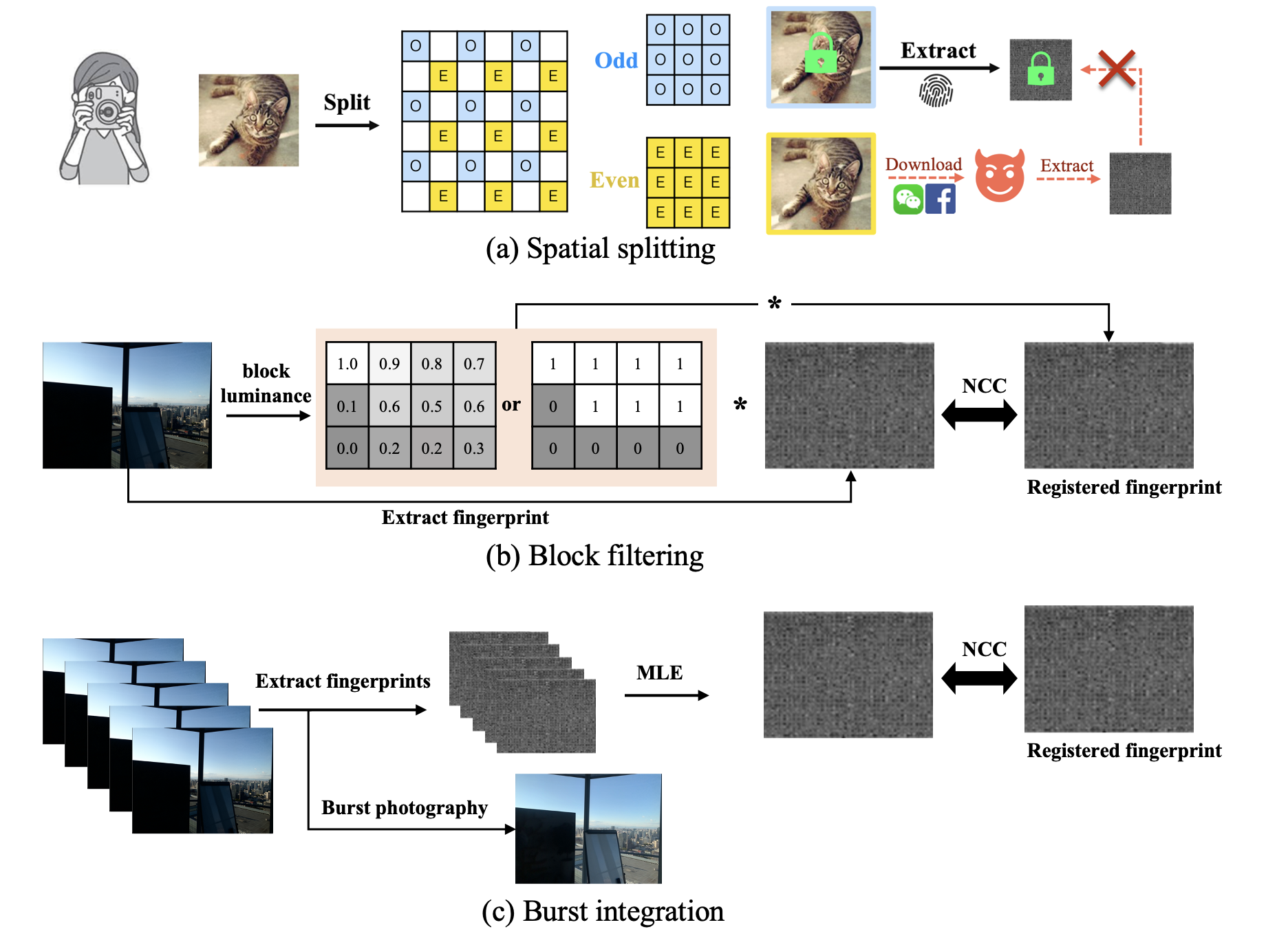} 
\caption{Illustration of sub-modules of (a)spatially splitting, (b)block filtering, and (c)burst integration.}
\label{fig3}
\end{figure}

\noindent\textbf{Computational cost reduction.} Sensor fingerprints are usually large in dimension, especially for millions of pixels in today's smartphones camera. This makes fingerprint process and matching slow due to the large computational cost, leading to impractical implementation for later cryptographic scheme in Section 5. Here, we adopt the binary quantization proposed in the work by Bayram et al. \cite{bayram2012efficient} to reduce storage requirements and computation time while still maintaining an acceptable matching accuracy. In detail, given a real-valued fingerprint $K^R$, the binary-quantized version $K^B$ is:
\begin{equation}
\begin{aligned}
  K^B_i=\begin{cases}
 +1, \ K^R_i \ge  0\\
-1,\ K^R_i <  0
\end{cases}
  \label{eq8}
\end{aligned}
\end{equation}
where $i$ is the pixel index on the fingerprint.
According to experiments in the work by Bayram et al. \cite{bayram2012efficient}, this binarization of sensor fingerprints can achieve 21 times speedup in loading to memory, and 9 times faster computation. In addition, we also provide a simplified version of fingerprint extraction procedure and put it into ZKP circuit, which will be introduced in Section 5. After incorporating these optimization sub-modules, the overall fingerprint extraction procedure is indicated in Figure 5. In real applications (e.g., copyright trading), we verify the public RAW even photo and its connection to the fingerprint. The downstream produced data of RAW even photo such as its irreversible compression (e.g., JPEG) can be verified using near-duplicate image matching methods which is beyond the scope of this paper.

\begin{figure}[htb]
\centering
\includegraphics[width=\linewidth]{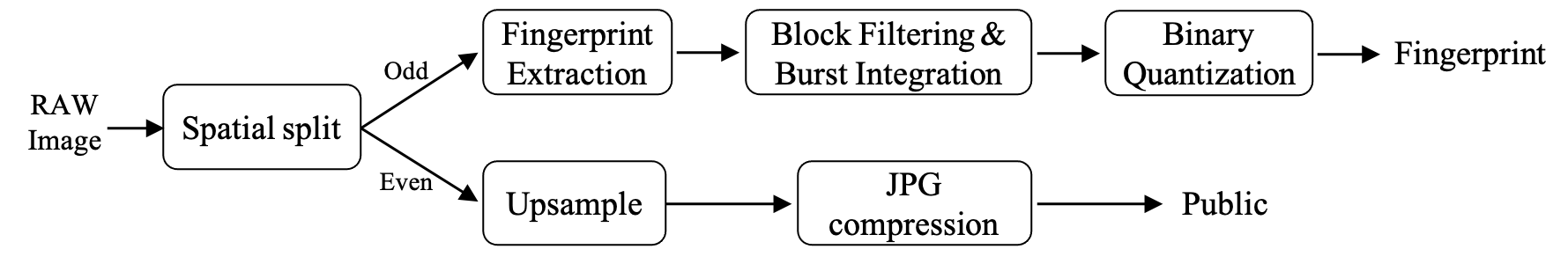} 
\caption{Overall fingerprint extraction procedure including optimization sub-modules.}
\label{fig3}
\end{figure}

\section{Source identification system}
In this section, we propose two source identification systems which integrates previously proposed fingerprint extraction network and optimization modules. 
We also incorporate cryptographic schemes to achieve the complete scheme design with higher reliability and security, so that it can be applied in real scenarios. 

Both these two practical schemes shown in Figure 6 contain three stages, i.e., registration stage for obtaining reliable device fingerprint with one or more photos as input, generation stage for taking one photo and uploading the photo together with identification script (signature or ZKP script) to public, verification stage for identifying the source camera of photo. Notably, registration stage only needs to be executed once for each device, and the verification stage can be executed anytime, anywhere and many times.

\begin{figure}[htb]
\centering
\includegraphics[width=\linewidth]{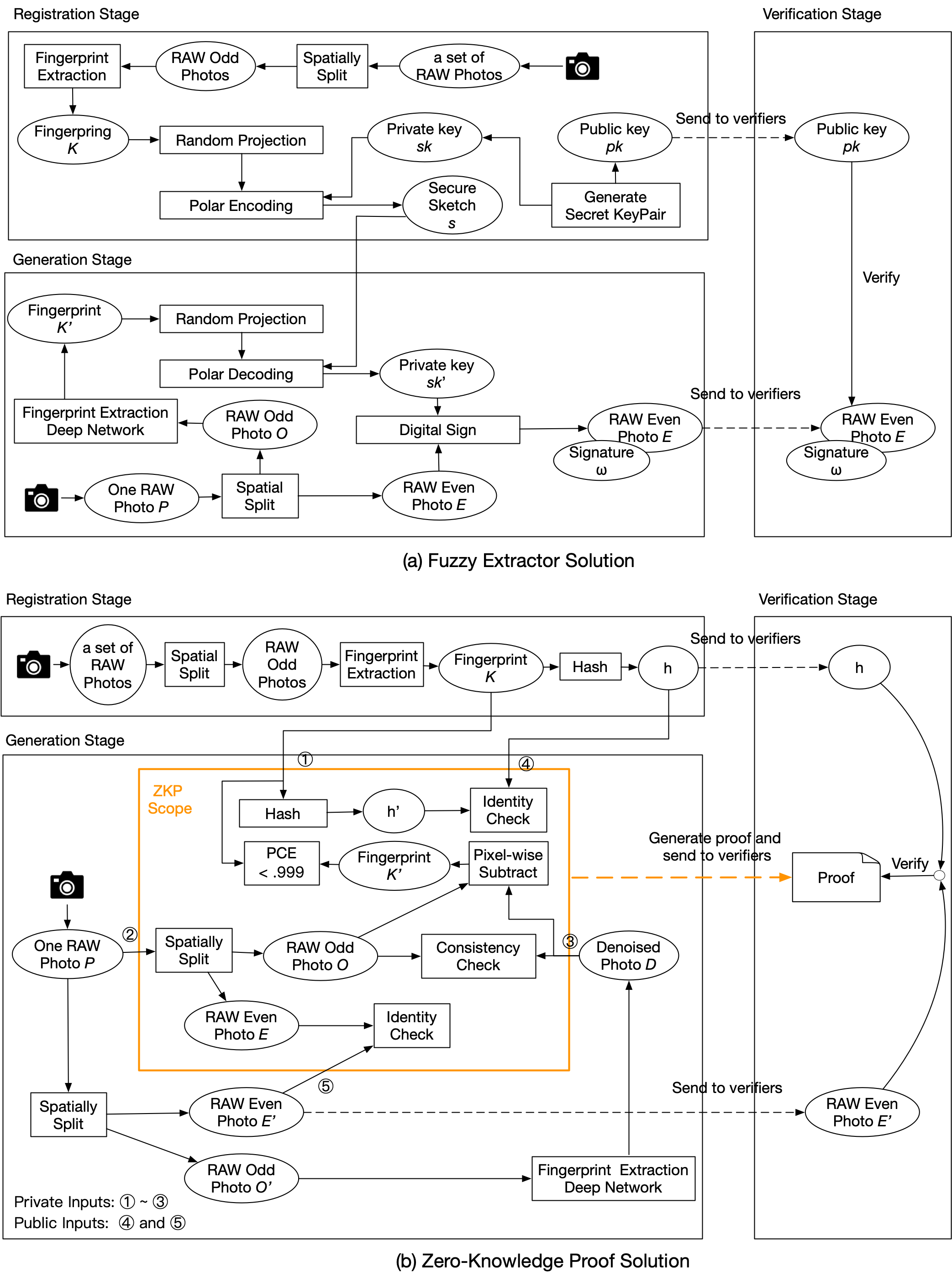} 
\caption{Two proposed source identification schemes based on fuzzy extractor and zero-knowledge proof respectively.}
\label{fig3}
\end{figure}

\subsection{Fuzzy extractor solution}
Our first solution is deeply inspired by the PRNU-based key management scheme presented by Valsesia et al. \cite{valsesia2017user}.  We present a PRNU based digital signature based authentication scheme. Our main idea is to use the camera fingerprint of a user’s device as a physical unclonable function (PUF), which enables a hide-and-recover scheme of user’s private key $sk$ in a great change to success if user’s private key $sk$ is encoded with  polar coding \cite{Arkan2008ChannelPA} and the user is capable to extract similar fingerprints from same device. 

In the registration stage, the system extracts the fingerprint from a certain number of photos (single photo or multiple photos registration) based on the fingerprint extraction method. In the extraction we use only the odd rows of photos to prevent information leakage as presented in Section 4. Instead of directly sending the fingerprint consisting millions of real numbers, the system first compresses it by previously mentioned binarization and random projection \cite{valsesia2015compressed}. The system also stores some side information related to the seed of the pseudo random number generator and the positions of the entries with largest magnitude (outliers) within those random projections, which will be then used in the generation stage. The exact algorithm as well as the role of the outliers will be made clear in the following sections. After that the compressed fingerprint is processed by a fuzzy extractor \cite{dodis2004fuzzy, valsesia2017user}. Namely, the system:

\begin{itemize}
\item Firstly, generate a key pair consisting of a private key $sk$ and public key $pk$, $(sk, pk) \gets \mathrm{KeyGen}(1^\lambda ) $, where $\lambda$ is the security parameter, here we let $\lambda$ be 128 for 128-bit security.
\item Next, generate a secure sketch $s$ of $sk$ , $s = K \oplus \mathbb{C} (sk)$, $\mathbb{C}$ denotes an $(m, \lambda)$ error correcting polar code where $m$ is bit length of fingerprint $K$. 
\item Then the system registers the public key $pk$ to the Verifier, stores the secure sketch $s$ publicly or locally and discards the private key $sk$.
\end{itemize}

In the generation stage, once the user takes a photo the system reproduces a fingerprint $K'$ by our proposed deep extraction network and compresses it using random projection \cite{valsesia2015compressed} according to the stored side information. The system then uses the fuzzy extractor scheme for reproducing the private key string from the compressed fingerprint of $K'$ and the secure sketch $s$, $sk' = \mathbb{D}(K' \oplus s)$, $\mathbb{D}$ denotes the decoding algorithm of the polar error correcting code \cite{arikan2009channel}. Then the system signs the RAW even photo with private key $sk$ to produce a signature $\omega$ using digital signature schemes such as ECDSA, SM2 \cite{johnson2001elliptic, martinkauppi2020design} etc. If the newly taken photo provides a version of the compressed fingerprint sufficiently close to the registered one, then the system can reproduce the same private key of the registration stage.

In the verification stage, the verifier verifies the signature $\omega$ with the received RAW even photo and public key $pk$. If the verification algorithm passes, it indicates that the reproduced private key is identical to the one discard in the registration phase; otherwise the reproduction of private key failed which means the two fingerprints are not close to each other.

\subsection{Zero-knowledge proof solution}
Different from signature based authentication scheme that the realization of source identification is based on the device owner to honestly signing photos from his own camera, our second zero-knowledge proof solution enables the device owner(prover) to convince all verifiers that the photos presented are from certain registered source camera if the prover is capable to produce a valid proof of the predetermined zero-knowledge argument.

In the registration stage, the prover again extracts the fingerprint from a set of photos based on the conventional PRNU extraction mehtod. Then the prover computes the digest $h = hash(K)$ of the device fingerprint $K$ via cryptographic secure hash function such as SHA256 \cite{rachmawati2018comparative}. The digest is sent to verifier as the identity of camera. The prover stores the device fingerprint $K$.

In the generation stage, once the user takes a photo from the source camera the prover proves that the photo is indeed taken from the camera with registered identity. To achieve such obligation, we introduce a ZKP solution, the solution consists of two roles, prover and verifier, where the prover wants to convince the verifier that some statements are true without revealing.

In our case, the statement is very complicated thus industrial ZKP protocol for general statements (e.g., zkSNARKS\cite{eliben19zksnarks}) is adopted. Our goal is to translate the process of generating matching digest of fingerprint from RAW photo into arithmetic statements and thus can be proved via zkSNARKs. 
The statements includes:
\begin{itemize}
\item The prover has a RAW photo $P$ that is spatially split into an RAW odd photo $O$ and an RAW even photo $E$.
\item the RAW even photo $E$ is identical to the photo $E'$ which will be sent to verifiers.
\item There exists a denoised photo $D$ that is consistent with the RAW odd photo $O$. The consistency check procedure ensures sufficient similarity of the low-level image features between two photos which we will describe in detail.
\item The pixel-wise subtraction of the two photos $O-D$ (i.e., reproduced fingerprint $K'$) is correlated with the registered fingerprint $K$.
\item The digest $h'$ (via e.g., SHA256) of a fingerprint $K'$ is identical to the registered digest $h$ (via e.g., SHA256). 
\end{itemize}

The overall statement can be summarized as follows:
\begin{equation}\label{eqn:zkpstatement}
\begin{aligned}
\Pi _{statement}= \ & \{P,O,E,E',D,K,K',h,h'\mid  \\
\ & P=\{O,E\},E\approx E' \\
\ & 1=\mathrm{CheckConsistency} (D,O), \\
\ & K'= (O-D)\approx K, h=h'=\mathrm{hash} (K)\}
\end{aligned}
\end{equation}
To prove the statement above, user has the private inputs including $K, P, D$ as witness and $h, E'$ as public inputs, the prove system outputs a proof script $ps$.

We would incur a large computational cost if we kept extracting the denoised image in ZKP circuit. To address this problem, we design a consistency checking procedure that excludes the heavy extraction network from ZKP circuit while approximately preserving the correctness and completeness. Assuming the denoised image is already extracted and passed as an input of ZKP circuit, the generation process confirms two necessary conditions: (1) the noise pattern (i.e., the original image subtract the denoised image) is correlated with the registered camera fingerprint; (2) and in addition, the denoised image is consistent with the original odd image. For the first condition, we use normalized cross-correlation to measure the correlation. For the second condition, we design a consistency checking procedure as show in Algorithm 1. 

\begin{algorithm}
\caption{Consistency Checking Procedure}

\textbf{Inputs:}\ odd image $O$ and denoised image $D$ \\
    Grid partition $O$ into $N$ disjoint patches $\{o_k | k \in 1..N\} $ \\
    Grid partition $D$ into $N$ disjoint patches $\{d_k | k \in 1..N\} $ \\
    $count = 0$ \\
    \For{each ($o_k$, $d_k$) /* $o_k$ and $d_k$ are in same location */} { 
        \uIf{$C1(o_k,d_k) \geq C1\_thld$ or $C2(o_k,d_k) \geq C2\_thld$} {
            $count = count + 1$ \ \ // $o_k$ is consistent with $d_k$
        }
    }
    
    \textbf{if} $count \geq count\_thld$ \textbf{then} \\
    \qquad \textbf{return} True \ \ // $D$ is consistent with $O$ \\
    \textbf{else} \textbf{return} False \ \ // $D$ is inconsistent with $O$
\end{algorithm}

In Algorithm 1, we first grid partition the odd image $O$ and the denoised image $D$ into disjoint patches (practically we use patch size of 128 $\times$ 128). Then for each patch pair $(o_k, d_k)$ with the same location, we calculate the values of consistency coefficients $C1$ and $C2$ which we define as follows.

\begin{equation}\label{eqn:c1}\begin{aligned}
C1(o_k,d_k) = Jaccard(& Threshold(Sobel(o_k)), \\ & \ Threshold(Sobel(d_k)))
\end{aligned}\end{equation}

\begin{equation}\label{eqn:c2}\begin{aligned}
C2(o_k,d_k) = max\{
& IoA(Threshold(Pool(o_k)), \\
& \hspace{2em} Threshold(Pool(Sobel(d_k)))), \\ 
& 1-IoA(Threshold(Pool(o_k)), \\
& \hspace{3.5em} Threshold(Pool(Sobel(d_k)))) 
\  \}
\end{aligned}\end{equation}

Given $Sobel$ is a sobel operator \cite{seger2012generalized} with kernel size of 3$\times$3, $Threshold$ is a thresholding operator \cite{ridler1978picture} with mean pixel value as the threshold, $Pool$ is a mean pooling operator \cite{mouton2021stride}
and $IoA$ is the intersection over pixel-wise area defined as follows (X and Y are two-dimensional binary matrices with the same size).

\begin{equation}\begin{aligned}
IoA(X,Y) = 1-\frac{|logical\_xor(X,Y)|}{|X|}
\end{aligned}\end{equation}

As described in Eq.\eqref{eqn:c1} and Eq.\eqref{eqn:c2}, the computational cost of $C1$ and $C2$ involves merely some linear operations which is much less than fingerprint extraction network. In Appendix D, we visualize the ability of consistency coefficients which shows $C1$ focuses on close-up consistency while $C2$ focuses on contour consistency. Collaboratively using $C1$ and $C2$ can detect almost all the near duplicate patches (i.e., image patches and their denoised version).

In the verification stage, the verifier receives the proof script $ps$, the RAW even photo $E'$ and the registered digest $h$ then checks the proofs for alleged statements. A successful verification of proof script indicates that either the device honestly takes the RAW photo $P$ (from which $E'$ is spatially splitted) from the registered source camera. If the verification failed, then it tells that the RAW even photo $E'$ and the source camera are not connected. In the next subsection, we analyze that forging the proof script is difficult.

\subsection{Security analysis}
We analyze security issues on our proposed solutions individually from cryptographic side. Different from performance analysis of fingerprint extraction algorithms in Section 6, cryptographic security protects our solution against attackers in real applications.

\textbf{Security of Fuzzy Extractor Solution:} Our fuzzy extractor solution works under an important assumption that attacker do not have access to the source camera, RAW odd photo, fingerprint $K$ and secret key $sk$. In order to forge a signature $\omega$, the attacker must be able to acquire either the secrete key $sk$ or the fingerprint $K$ extracted from RAW odd photo. As we prove in Appendix C, probability of success of this attack can be upper bounded as follows.

\begin{theorem}
If an attacker do not have access to the source camera, RAW odd photo $O$, fingerprint $K$ and private key $sk$, then the probability for the attacker to successfully forge a cryptographic secure signature (e.g., ECDSA, SM2 etc.,) with public key $pk$ is $P_a \leq \frac{1}{2^{\lambda-1}}$ where $\lambda$ is the security parameter.
\end{theorem}

\textbf{Security of Zero-knowledge Proof Solution:} Our ZKP solution works under an important assumption that attacker do not have access to the fingerprint $K$ which is secretly protected by the device. Recall that our ZKP solution requires the prover to prove the statement \eqref{eqn:zkpstatement}, an attacker must be able to either forge the public inputs that complies with the statement or convince the verifier of a false statement. The latter indicates that the attacker is able to break the soundness property of underlying ZKP system, which is beyond the scope of this paper. As we prove in Appendix C, probability of success of this attack can be upper bounded as follows.

\begin{theorem}
Let $hash(\cdot)$ be a cryptographic secure hash function (e.g., SHA256, SHA3 \cite{dworkin2015sha} etc.), if the attacker do not have access to fingerprint $K$ and can not break the pre-image resistance property of $hash(\cdot)$ \cite{rogaway2004cryptographic}, then the attacker can forge a prove of statement \eqref{eqn:zkpstatement} with probability $P_b \leq \frac{1}{2^m}+\frac{1}{2^{2\cdot \lambda}}$ where $m$ is the bit length of fingerprint $K$,  $\lambda$ is the security parameter and a pre-image here refers to the message mapped to a particular digest via hashing.  
\end{theorem}

\section{Experiments}
\subsection{Implementation details}

\textbf{Dataset and metric.} As mentioned in Section 4, in order to avoid fingerprint leakage, we propose to utilize RAW images rather than JPEG images for fingerprint extraction and matching. However, there is no large-scale RAW photo dataset for training stable fingerprints. Therefore, we collect a large amount of RAW photos taken by iPhones, consisting of over 150,000 images and 72 cameras. Among them, we select 1,665 photos taken by 15 different cameras as the benchmark test set, and ensure that the camera devices in the benchmark set do not exist in the training set. We train the fingerprint extraction network on the splitted training part of this RAW dataset, and benchmark our proposed algorithm with the test set. We utilize normalized cross-correlation as the similarity measure for camera identification. As for method comparison metric, we adopt AUC (Area Under Curve, higher is better) \cite{lobo2008auc} and EER (Equal Error Rate, lower is better) for performance. 

\noindent\textbf{Network details.} The fingerprint extraction network is trained with RAW photos under guidance of triplet loss and frequency loss. In order to have a fair comparison with previous works \cite{kirchner2019spn,cozzolino2019noiseprint}, we select DnCNN \cite{zhang2017beyond} as backbone denoise network. The pre-computed PRNU in Figure 2 is extracted with wavelet-based denoiser from 40 flat images. We mine hardest positive sample and hardest negative sample per anchor within batch size of 2048, and triplet margin is 0.2. We train the network using Adam optimizer \cite{kingma2014adam}, learning rate of 1e-5, weight decay of 1e-6, and 100 epochs. 

\subsection{Network performance}

First, we ablate different settings of network components in Table 1. We derive one fingerprint from each RAW photo as the device registered fingerprint, and AUC and EER are calculated from the correlation matrix between these single image camera fingerprints, i.e., a 1,665$\times$1,665 matrix, which can directly reflect the model performance. Here, the basic denoise model has the same parameters with the pretrained denoise model in the work by Zhang et al. \cite{zhang2017beyond}. Compared with residual noise performance directly from pretrained denoise model, AUC is significantly improved after incorporating deep metric learning with hard mining strategy. Furthermore, the supervised guidance of pre-computed PRNU brings considerable performance gains, and the frequency domain loss slightly outperforms spatial domain loss, which is consistent with the observation in Figure 3. Meanwhile, pixel-shuffle operator also slightly improve the network performance. At last, we also utilize the post-processing approach of zero mean (ZM) and wiener filter (WF) proposed in PRNU algorithm \cite{chen2008determining} to further improve the result.

\begin{table}[]
    \centering
    \caption{Ablation of fingerprint extraction network.}
    \label{table2}
    \begin{tabular}{l||c|c}
        \toprule
        {Method} & AUC $\uparrow$ & EER $\downarrow$  \\ \toprule
        {basic denoise model} & 54.29\%  & 47.31\% \\ 
        { + Triplet loss} & 88.82\% & 18.53\% \\ 
        { + Spatial consistency} & 99.58\%  & 3.003\%  \\ 
        { + Pixel Shuffle} & 99.73\% & 2.645\% \\
        { + Frequency consistency} &{99.75\%}  & {2.345\%}  \\
        { + Postprocessing(ZM \& WF)} & {99.80\%}  & {1.656\%} \\  
        {Full model} & \textbf{99.80\%}  & \textbf{1.656\%} \\ \bottomrule
    \end{tabular}
\end{table}

We compare our proposed algorithm to two open-sourced camera fingerprints, i.e., wavelet denoiser based PRNU \cite{chen2008determining} and CNN-based camera model fingerprint (Noiseprint) \cite{cozzolino2019noiseprint} on benchmark dataset. Here, we test two scenarios: single photo fingerprint registration and multiple photos fingerprint registration, with $N = 1$ and $N = 40$ in Eq.(3) for registered fingerprint estimation respectively. Table 2 shows our algorithm outperforms PRNU and Noiseprint by a large margin on benchmark dataset with much higher AUC and lower EER. Figure 7(a-d) give some insights by plotting average correlation scores as confusion matrix of all 15 camera devices. Results of our proposed network show significantly better discrimination between positive pair and negative pair in comparison with others. ROC curves shown in Figure 7(e) also indicates our best performance among all the benchmarked algorithms.

\begin{table}[h!]
\caption{Fingerprint accuracy performance comparison of ours with previously open-sourced fingerprint extraction algorithms on iPhone RAW photos. Result with $^*$ indicates containing post-processing (ZM \& WF).}
\centering
\label{table1}
\scalebox{1}{
\begin{tabular}{c|l||c|c|c|c}
\toprule
{Register} & {Method}   & AUC $\uparrow$ & EER $\downarrow$  & AUC$^*$ $\uparrow$ & EER$^*$ $\downarrow$ \\ \toprule
\multirow{3}{*}{\begin{tabular}[c]{@{}c@{}}Single\end{tabular}}    & PRNU        & 63.23\%  & 40.62\% & 99.33\%  & 3.513\%\\
                                                                             & Noiseprint        & 53.10\%  & 47.75\%  & 63.23\%  & 40.62\%   \\  
                                                                             & Ours        & \textbf{99.75\%}  & \textbf{2.345\%}   & \textbf{99.80\%}  & \textbf{1.656\%} \\ 
                                                                            
                                        \midrule                                 
\multirow{3}{*}{\begin{tabular}[c]{@{}c@{}}Multiple\end{tabular}} & PRNU        & 65.14\%  & 37.20\% & 99.99\%  & 0.013\% \\ 
                                                                             & Noiseprint        & 50.66\%  & 48.30\%  & 51.43\%  & 49.40\%   \\  
                                                                             & Ours        & \textbf{99.95\%}  & \textbf{0.708\%}   & \textbf{100.0\%}  & \textbf{0.0\%}  \\  
                                                                             
                                                                         \bottomrule
\end{tabular}}
\end{table}

\begin{figure}[htb]
\centering
\includegraphics[width=0.7\linewidth]{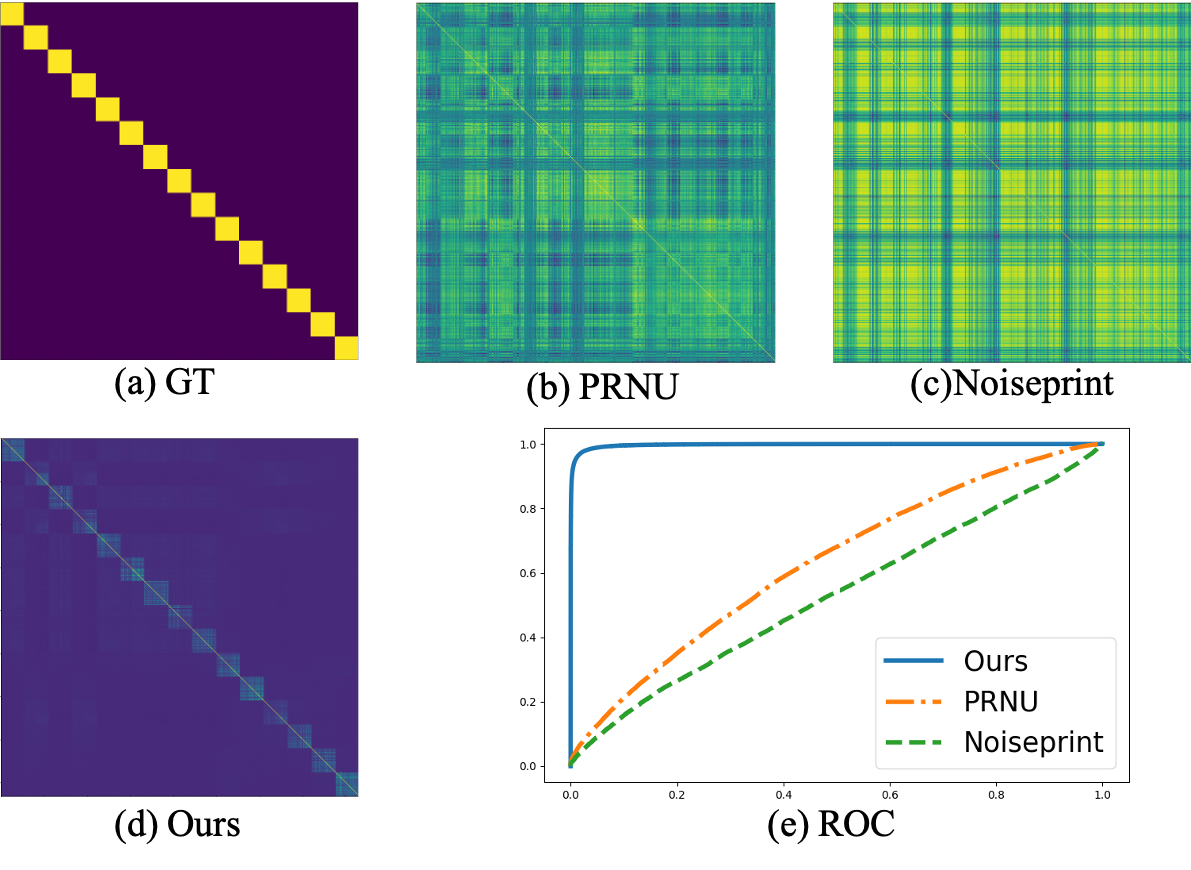} 
\caption{(a-d)Correlation scores between extracted fingerprints of ground truth (GT) and different methods. (d) ROC curves of the compared fingerprint extraction methods.}
\label{fig7}
\
\end{figure}

\subsection{Fingerprint optimization}

As mentioned in Section 4, we have proposed some optimization sub-modules to improve the accuracy of fingerprint extraction. First, we verify the effectiveness of block filtering with different block size and filter weight mask. 
The best result is achieved with block size of 64 and fixed percentage selection on luminance of 50\%, and this block filtering approach works effectively not only on our proposed method, but also on PRNU method. Another proposed accuracy improvement module is burst integration, and our test set for this benchmark with 1,665 photos consists of exactly 555 sets taken in three burst photography. Therefore, we can directly use the three-in-one estimation to generate fingerprints and calculate correlation. After optimization with block filtering and burst integration, our proposed network based fingerprint can achieve AUC = 99.99\% on single image registration with almost no error rate on the benchmark dataset (baseline without optimization in Table 2 is 99.75\%). We plot the histograms of all the positive and negative correlations before and after optimization sub-modules in Figure 8 (a) and (b) respectively. There exist some overlaps between the correlation distributions of positive and negative samples in Figure 8(a). But after the optimization sub-modules, the correlations of positive and negative samples are completely separated. At last, we also observe the fingerprint performance after binary quantization in Figure 8(c), and it maintains an acceptable matching accuracy with AUC = 99.98\% but with much less computational cost.

\begin{figure}[htb]
\centering
\includegraphics[width=\linewidth]{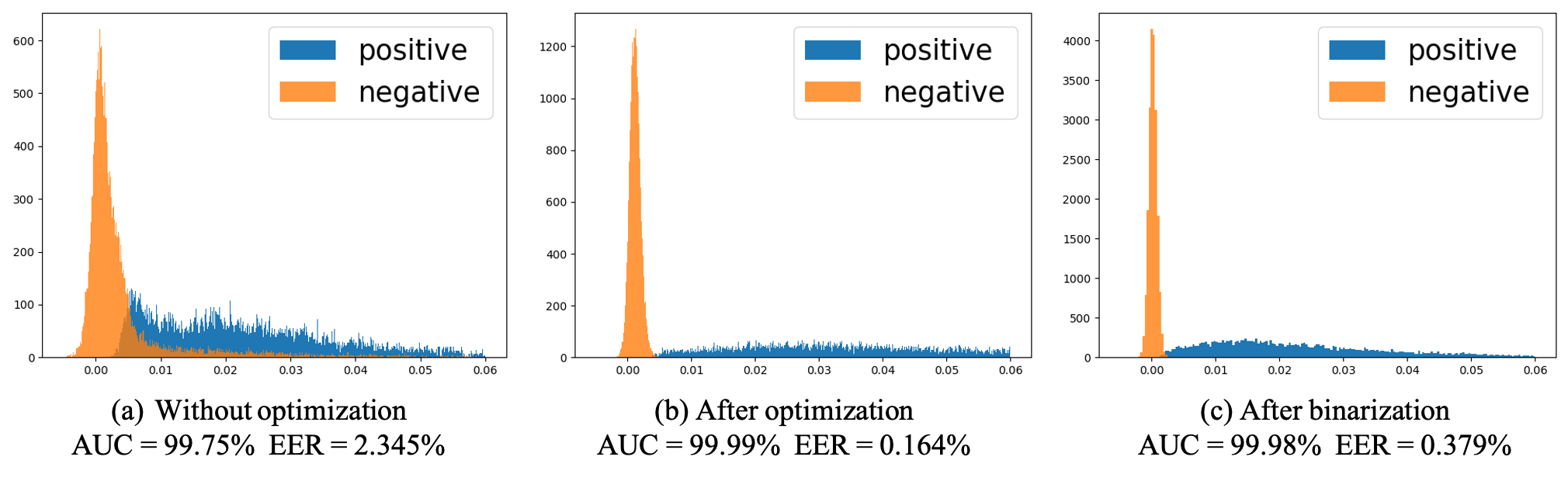} 
\caption{Histogram of correlation scores from same camera (positive) and different camera (negative) (a)before optimization, (b)after optimization, (c)after binary quantization.}
\label{fig7}
\end{figure}

\section{Conclusions}
This paper presents a reliable source camera identification framework for web photos. In detail, we firstly introduce a neural enhanced camera fingerprint extraction algorithm and demonstrate it strong performance. Then several general sub-modules are proposed to further optimize the system on both performance and computational efficiency. Finally for practical realization, two cryptographic schemes are incorporated to achieve the complete scheme design with higher reliability and security. We hope our new perspective will pave a way towards a new paradigm for accurate and practical source camera identification.

\begin{acks}
This work is partly supported by R\&D Program of DCI Technology and Application Joint Laboratory.
\end{acks}

\bibliographystyle{ACM-Reference-Format}
\bibliography{sample-base}


\begin{thebibliography}{62}


\ifx \showCODEN    \undefined \def \showCODEN     #1{\unskip}     \fi
\ifx \showDOI      \undefined \def \showDOI       #1{#1}\fi
\ifx \showISBNx    \undefined \def \showISBNx     #1{\unskip}     \fi
\ifx \showISBNxiii \undefined \def \showISBNxiii  #1{\unskip}     \fi
\ifx \showISSN     \undefined \def \showISSN      #1{\unskip}     \fi
\ifx \showLCCN     \undefined \def \showLCCN      #1{\unskip}     \fi
\ifx \shownote     \undefined \def \shownote      #1{#1}          \fi
\ifx \showarticletitle \undefined \def \showarticletitle #1{#1}   \fi
\ifx \showURL      \undefined \def \showURL       {\relax}        \fi
\providecommand\bibfield[2]{#2}
\providecommand\bibinfo[2]{#2}
\providecommand\natexlab[1]{#1}
\providecommand\showeprint[2][]{arXiv:#2}

\bibitem[Arikan(2009)]%
        {arikan2009channel}
\bibfield{author}{\bibinfo{person}{Erdal Arikan}.}
  \bibinfo{year}{2009}\natexlab{}.
\newblock \showarticletitle{Channel polarization: A method for constructing
  capacity-achieving codes for symmetric binary-input memoryless channels}.
\newblock \bibinfo{journal}{\emph{IEEE Transactions on information Theory}}
  \bibinfo{volume}{55}, \bibinfo{number}{7} (\bibinfo{year}{2009}),
  \bibinfo{pages}{3051--3073}.
\newblock


\bibitem[Arıkan(2008)]%
        {Arkan2008ChannelPA}
\bibfield{author}{\bibinfo{person}{Erdal Arıkan}.}
  \bibinfo{year}{2008}\natexlab{}.
\newblock \showarticletitle{Channel Polarization: A Method for Constructing
  Capacity-Achieving Codes for Symmetric Binary-Input Memoryless Channels}.
\newblock \bibinfo{journal}{\emph{IEEE Transactions on Information Theory}}
  \bibinfo{volume}{55} (\bibinfo{year}{2008}), \bibinfo{pages}{3051--3073}.
\newblock


\bibitem[Ba et~al\mbox{.}(2018)]%
        {ba2018abc}
\bibfield{author}{\bibinfo{person}{Zhongjie Ba}, \bibinfo{person}{Sixu Piao},
  \bibinfo{person}{Xinwen Fu}, \bibinfo{person}{Dimitrios Koutsonikolas},
  \bibinfo{person}{Aziz Mohaisen}, {and} \bibinfo{person}{Kui Ren}.}
  \bibinfo{year}{2018}\natexlab{}.
\newblock \showarticletitle{ABC: Enabling smartphone authentication with
  built-in camera}. In \bibinfo{booktitle}{\emph{25th Annual Network and
  Distributed System Security Symposium, NDSS 2018}}.
\newblock


\bibitem[Bayram et~al\mbox{.}(2012)]%
        {bayram2012efficient}
\bibfield{author}{\bibinfo{person}{Sevin{\c{c}} Bayram},
  \bibinfo{person}{H{\"u}srev~Taha Sencar}, {and} \bibinfo{person}{Nasir
  Memon}.} \bibinfo{year}{2012}\natexlab{}.
\newblock \showarticletitle{Efficient sensor fingerprint matching through
  fingerprint binarization}.
\newblock \bibinfo{journal}{\emph{IEEE Transactions on Information Forensics
  and Security}} \bibinfo{volume}{7}, \bibinfo{number}{4}
  (\bibinfo{year}{2012}), \bibinfo{pages}{1404--1413}.
\newblock


\bibitem[Ben-Sasson et~al\mbox{.}(2014)]%
        {ben2014succinct}
\bibfield{author}{\bibinfo{person}{Eli Ben-Sasson}, \bibinfo{person}{Alessandro
  Chiesa}, \bibinfo{person}{Eran Tromer}, {and} \bibinfo{person}{Madars
  Virza}.} \bibinfo{year}{2014}\natexlab{}.
\newblock \showarticletitle{Succinct $\{$Non-Interactive$\}$ zero knowledge for
  a von neumann architecture}. In \bibinfo{booktitle}{\emph{23rd USENIX
  Security Symposium (USENIX Security 14)}}. \bibinfo{pages}{781--796}.
\newblock


\bibitem[Cao et~al\mbox{.}(2020)]%
        {DBLP:journals/tii/CaoJM20}
\bibfield{author}{\bibinfo{person}{Yan Cao}, \bibinfo{person}{Feng Jia}, {and}
  \bibinfo{person}{Gunasekaran Manogaran}.} \bibinfo{year}{2020}\natexlab{}.
\newblock \showarticletitle{Efficient Traceability Systems of Steel Products
  Using Blockchain-Based Industrial Internet of Things}.
\newblock \bibinfo{journal}{\emph{{IEEE} Trans. Ind. Informatics}}
  \bibinfo{volume}{16}, \bibinfo{number}{9} (\bibinfo{year}{2020}),
  \bibinfo{pages}{6004--6012}.
\newblock
\urldef\tempurl%
\url{https://doi.org/10.1109/TII.2019.2942211}
\showDOI{\tempurl}


\bibitem[Chen et~al\mbox{.}(2008)]%
        {chen2008determining}
\bibfield{author}{\bibinfo{person}{Mo Chen}, \bibinfo{person}{Jessica
  Fridrich}, \bibinfo{person}{Miroslav Goljan}, {and} \bibinfo{person}{Jan
  Luk{\'a}s}.} \bibinfo{year}{2008}\natexlab{}.
\newblock \showarticletitle{Determining image origin and integrity using sensor
  noise}.
\newblock \bibinfo{journal}{\emph{IEEE Transactions on information forensics
  and security}} \bibinfo{volume}{3}, \bibinfo{number}{1}
  (\bibinfo{year}{2008}), \bibinfo{pages}{74--90}.
\newblock


\bibitem[Cozzolino et~al\mbox{.}(2015)]%
        {cozzolino2015splicebuster}
\bibfield{author}{\bibinfo{person}{Davide Cozzolino}, \bibinfo{person}{Giovanni
  Poggi}, {and} \bibinfo{person}{Luisa Verdoliva}.}
  \bibinfo{year}{2015}\natexlab{}.
\newblock \showarticletitle{Splicebuster: A new blind image splicing detector}.
  In \bibinfo{booktitle}{\emph{2015 IEEE International Workshop on Information
  Forensics and Security (WIFS)}}. IEEE, \bibinfo{pages}{1--6}.
\newblock


\bibitem[Cozzolino and Verdoliva(2019)]%
        {cozzolino2019noiseprint}
\bibfield{author}{\bibinfo{person}{Davide Cozzolino} {and}
  \bibinfo{person}{Luisa Verdoliva}.} \bibinfo{year}{2019}\natexlab{}.
\newblock \showarticletitle{Noiseprint: a CNN-based camera model fingerprint}.
\newblock \bibinfo{journal}{\emph{IEEE Transactions on Information Forensics
  and Security}}  \bibinfo{volume}{15} (\bibinfo{year}{2019}),
  \bibinfo{pages}{144--159}.
\newblock


\bibitem[Delbracio et~al\mbox{.}(2021)]%
        {DBLP:journals/corr/abs-2102-09000}
\bibfield{author}{\bibinfo{person}{Mauricio Delbracio}, \bibinfo{person}{Damien
  Kelly}, \bibinfo{person}{Michael~S. Brown}, {and} \bibinfo{person}{Peyman
  Milanfar}.} \bibinfo{year}{2021}\natexlab{}.
\newblock \showarticletitle{Mobile Computational Photography: {A} Tour}.
\newblock \bibinfo{journal}{\emph{CoRR}}  \bibinfo{volume}{abs/2102.09000}
  (\bibinfo{year}{2021}).
\newblock
\showeprint[arXiv]{2102.09000}
\urldef\tempurl%
\url{https://arxiv.org/abs/2102.09000}
\showURL{%
\tempurl}


\bibitem[Dodis et~al\mbox{.}(2004)]%
        {dodis2004fuzzy}
\bibfield{author}{\bibinfo{person}{Yevgeniy Dodis}, \bibinfo{person}{Leonid
  Reyzin}, {and} \bibinfo{person}{Adam Smith}.}
  \bibinfo{year}{2004}\natexlab{}.
\newblock \showarticletitle{Fuzzy extractors: How to generate strong keys from
  biometrics and other noisy data}. In \bibinfo{booktitle}{\emph{International
  conference on the theory and applications of cryptographic techniques}}.
  Springer, \bibinfo{pages}{523--540}.
\newblock


\bibitem[Dworkin et~al\mbox{.}(2015)]%
        {dworkin2015sha}
\bibfield{author}{\bibinfo{person}{Morris~J Dworkin} {et~al\mbox{.}}}
  \bibinfo{year}{2015}\natexlab{}.
\newblock \showarticletitle{SHA-3 standard: Permutation-based hash and
  extendable-output functions}.
\newblock  (\bibinfo{year}{2015}).
\newblock


\bibitem[Eli Ben-Sasson and Virza(2014)]%
        {eliben19zksnarks}
\bibfield{author}{\bibinfo{person}{Eran~Tromer Eli Ben-Sasson,
  Alessandro~Chiesa} {and} \bibinfo{person}{Madars Virza}.}
  \bibinfo{year}{2014}\natexlab{}.
\newblock \showarticletitle{Succinct Non-Interactive Zero Knowledge for a von
  Neumann Architecture}.
\newblock \bibinfo{journal}{\emph{23rd USENIX Security Symposium}}
  (\bibinfo{year}{2014}).
\newblock


\bibitem[Farid and Lyu(2003)]%
        {farid2003higher}
\bibfield{author}{\bibinfo{person}{Hany Farid} {and} \bibinfo{person}{Siwei
  Lyu}.} \bibinfo{year}{2003}\natexlab{}.
\newblock \showarticletitle{Higher-order wavelet statistics and their
  application to digital forensics}. In \bibinfo{booktitle}{\emph{2003
  Conference on computer vision and pattern recognition workshop}},
  Vol.~\bibinfo{volume}{8}. IEEE, \bibinfo{pages}{94--94}.
\newblock


\bibitem[Ferreira et~al\mbox{.}(2021a)]%
        {DBLP:journals/data/FerreiraAC21}
\bibfield{author}{\bibinfo{person}{Sara Ferreira}, \bibinfo{person}{M{\'{a}}rio
  Antunes}, {and} \bibinfo{person}{Manuel~Eduardo Correia}.}
  \bibinfo{year}{2021}\natexlab{a}.
\newblock \showarticletitle{A Dataset of Photos and Videos for Digital
  Forensics Analysis Using Machine Learning Processing}.
\newblock \bibinfo{journal}{\emph{Data}} \bibinfo{volume}{6},
  \bibinfo{number}{8} (\bibinfo{year}{2021}), \bibinfo{pages}{87}.
\newblock
\urldef\tempurl%
\url{https://doi.org/10.3390/data6080087}
\showDOI{\tempurl}


\bibitem[Ferreira et~al\mbox{.}(2021b)]%
        {DBLP:journals/jimaging/FerreiraAC21}
\bibfield{author}{\bibinfo{person}{Sara Ferreira}, \bibinfo{person}{M{\'{a}}rio
  Antunes}, {and} \bibinfo{person}{Manuel~Eduardo Correia}.}
  \bibinfo{year}{2021}\natexlab{b}.
\newblock \showarticletitle{Exposing Manipulated Photos and Videos in Digital
  Forensics Analysis}.
\newblock \bibinfo{journal}{\emph{J. Imaging}} \bibinfo{volume}{7},
  \bibinfo{number}{7} (\bibinfo{year}{2021}), \bibinfo{pages}{102}.
\newblock
\urldef\tempurl%
\url{https://doi.org/10.3390/jimaging7070102}
\showDOI{\tempurl}


\bibitem[Fridrich(2013)]%
        {fridrich2013sensor}
\bibfield{author}{\bibinfo{person}{Jessica Fridrich}.}
  \bibinfo{year}{2013}\natexlab{}.
\newblock \showarticletitle{Sensor defects in digital image forensic}.
\newblock In \bibinfo{booktitle}{\emph{Digital image forensics}}.
  \bibinfo{publisher}{Springer}, \bibinfo{pages}{179--218}.
\newblock


\bibitem[Galbi(2003)]%
        {DBLP:journals/corr/cs-CY-0311054}
\bibfield{author}{\bibinfo{person}{Douglas~A. Galbi}.}
  \bibinfo{year}{2003}\natexlab{}.
\newblock \showarticletitle{Copyright and Creativity: Authors and
  Photographers}.
\newblock \bibinfo{journal}{\emph{CoRR}}  \bibinfo{volume}{cs.CY/0311054}
  (\bibinfo{year}{2003}).
\newblock
\urldef\tempurl%
\url{http://arxiv.org/abs/cs/0311054}
\showURL{%
\tempurl}


\bibitem[Galbraith and Gaudry(2016)]%
        {galbraith2016recent}
\bibfield{author}{\bibinfo{person}{Steven~D Galbraith} {and}
  \bibinfo{person}{Pierrick Gaudry}.} \bibinfo{year}{2016}\natexlab{}.
\newblock \showarticletitle{Recent progress on the elliptic curve discrete
  logarithm problem}.
\newblock \bibinfo{journal}{\emph{Designs, Codes and Cryptography}}
  \bibinfo{volume}{78}, \bibinfo{number}{1} (\bibinfo{year}{2016}),
  \bibinfo{pages}{51--72}.
\newblock


\bibitem[Galea and Farrugia(2018)]%
        {DBLP:journals/tifs/GaleaF18}
\bibfield{author}{\bibinfo{person}{Christian Galea} {and}
  \bibinfo{person}{Reuben~A. Farrugia}.} \bibinfo{year}{2018}\natexlab{}.
\newblock \showarticletitle{Matching Software-Generated Sketches to Face
  Photographs With a Very Deep CNN, Morphed Faces, and Transfer Learning}.
\newblock \bibinfo{journal}{\emph{{IEEE} Trans. Inf. Forensics Secur.}}
  \bibinfo{volume}{13}, \bibinfo{number}{6} (\bibinfo{year}{2018}),
  \bibinfo{pages}{1421--1431}.
\newblock
\urldef\tempurl%
\url{https://doi.org/10.1109/TIFS.2017.2788002}
\showDOI{\tempurl}


\bibitem[Ge(2018)]%
        {ge2018deep}
\bibfield{author}{\bibinfo{person}{Weifeng Ge}.}
  \bibinfo{year}{2018}\natexlab{}.
\newblock \showarticletitle{Deep metric learning with hierarchical triplet
  loss}. In \bibinfo{booktitle}{\emph{Proceedings of the European Conference on
  Computer Vision (ECCV)}}. \bibinfo{pages}{269--285}.
\newblock


\bibitem[Gloe et~al\mbox{.}(2012)]%
        {gloe2012unexpected}
\bibfield{author}{\bibinfo{person}{Thomas Gloe}, \bibinfo{person}{Stefan
  Pfennig}, {and} \bibinfo{person}{Matthias Kirchner}.}
  \bibinfo{year}{2012}\natexlab{}.
\newblock \showarticletitle{Unexpected artefacts in PRNU-based camera
  identification: A'Dresden Image Database'case-study}.
\newblock In \bibinfo{booktitle}{\emph{Proceedings of the on Multimedia and
  security}}. \bibinfo{pages}{109--114}.
\newblock


\bibitem[Goldreich and Oren(1994)]%
        {goldreich1994definitions}
\bibfield{author}{\bibinfo{person}{Oded Goldreich} {and} \bibinfo{person}{Yair
  Oren}.} \bibinfo{year}{1994}\natexlab{}.
\newblock \showarticletitle{Definitions and properties of zero-knowledge proof
  systems}.
\newblock \bibinfo{journal}{\emph{Journal of Cryptology}} \bibinfo{volume}{7},
  \bibinfo{number}{1} (\bibinfo{year}{1994}), \bibinfo{pages}{1--32}.
\newblock


\bibitem[Goljan et~al\mbox{.}(2009)]%
        {goljan2009large}
\bibfield{author}{\bibinfo{person}{Miroslav Goljan}, \bibinfo{person}{Jessica
  Fridrich}, {and} \bibinfo{person}{Tom{\'a}{\v{s}} Filler}.}
  \bibinfo{year}{2009}\natexlab{}.
\newblock \showarticletitle{Large scale test of sensor fingerprint camera
  identification}. In \bibinfo{booktitle}{\emph{Media forensics and security}},
  Vol.~\bibinfo{volume}{7254}. SPIE, \bibinfo{pages}{170--181}.
\newblock


\bibitem[Gou et~al\mbox{.}(2007)]%
        {gou2007noise}
\bibfield{author}{\bibinfo{person}{Hongmei Gou}, \bibinfo{person}{Ashwin
  Swaminathan}, {and} \bibinfo{person}{Min Wu}.}
  \bibinfo{year}{2007}\natexlab{}.
\newblock \showarticletitle{Noise features for image tampering detection and
  steganalysis}. In \bibinfo{booktitle}{\emph{2007 IEEE International
  Conference on Image Processing}}, Vol.~\bibinfo{volume}{6}. IEEE,
  \bibinfo{pages}{VI--97}.
\newblock


\bibitem[Groth(2016)]%
        {groth2016size}
\bibfield{author}{\bibinfo{person}{Jens Groth}.}
  \bibinfo{year}{2016}\natexlab{}.
\newblock \showarticletitle{On the size of pairing-based non-interactive
  arguments}. In \bibinfo{booktitle}{\emph{Annual international conference on
  the theory and applications of cryptographic techniques}}. Springer,
  \bibinfo{pages}{305--326}.
\newblock


\bibitem[He et~al\mbox{.}(2012)]%
        {he2012digital}
\bibfield{author}{\bibinfo{person}{Zhongwei He}, \bibinfo{person}{Wei Lu},
  \bibinfo{person}{Wei Sun}, {and} \bibinfo{person}{Jiwu Huang}.}
  \bibinfo{year}{2012}\natexlab{}.
\newblock \showarticletitle{Digital image splicing detection based on Markov
  features in DCT and DWT domain}.
\newblock \bibinfo{journal}{\emph{Pattern recognition}} \bibinfo{volume}{45},
  \bibinfo{number}{12} (\bibinfo{year}{2012}), \bibinfo{pages}{4292--4299}.
\newblock


\bibitem[Hermans et~al\mbox{.}(2017)]%
        {hermans2017defense}
\bibfield{author}{\bibinfo{person}{Alexander Hermans}, \bibinfo{person}{Lucas
  Beyer}, {and} \bibinfo{person}{Bastian Leibe}.}
  \bibinfo{year}{2017}\natexlab{}.
\newblock \showarticletitle{In defense of the triplet loss for person
  re-identification}.
\newblock \bibinfo{journal}{\emph{arXiv preprint arXiv:1703.07737}}
  (\bibinfo{year}{2017}).
\newblock


\bibitem[Igarashi et~al\mbox{.}(2021)]%
        {DBLP:conf/blockchain2/IgarashiKKKD21}
\bibfield{author}{\bibinfo{person}{Tatsuya Igarashi},
  \bibinfo{person}{Takabayashi Kazuhiko}, \bibinfo{person}{Yoshiyuki
  Kobayashi}, \bibinfo{person}{Hiroshi Kuno}, {and} \bibinfo{person}{Eric
  Diehl}.} \bibinfo{year}{2021}\natexlab{}.
\newblock \showarticletitle{Photrace: {A} Blockchain-Based Traceability System
  for Photographs on the Internet}. In \bibinfo{booktitle}{\emph{2021 {IEEE}
  International Conference on Blockchain, Blockchain 2021, Melbourne,
  Australia, December 6-8, 2021}}, \bibfield{editor}{\bibinfo{person}{Yang
  Xiang}, \bibinfo{person}{Ziyuan Wang}, \bibinfo{person}{Honggang Wang}, {and}
  \bibinfo{person}{Valtteri Niemi}} (Eds.). \bibinfo{publisher}{{IEEE}},
  \bibinfo{pages}{590--596}.
\newblock
\urldef\tempurl%
\url{https://doi.org/10.1109/Blockchain53845.2021.00089}
\showDOI{\tempurl}


\bibitem[Jiang et~al\mbox{.}(2021)]%
        {jiang2021focal}
\bibfield{author}{\bibinfo{person}{Liming Jiang}, \bibinfo{person}{Bo Dai},
  \bibinfo{person}{Wayne Wu}, {and} \bibinfo{person}{Chen~Change Loy}.}
  \bibinfo{year}{2021}\natexlab{}.
\newblock \showarticletitle{Focal frequency loss for image reconstruction and
  synthesis}. In \bibinfo{booktitle}{\emph{Proceedings of the IEEE/CVF
  International Conference on Computer Vision}}. \bibinfo{pages}{13919--13929}.
\newblock


\bibitem[Johnson et~al\mbox{.}(2001)]%
        {johnson2001elliptic}
\bibfield{author}{\bibinfo{person}{Don Johnson}, \bibinfo{person}{Alfred
  Menezes}, {and} \bibinfo{person}{Scott Vanstone}.}
  \bibinfo{year}{2001}\natexlab{}.
\newblock \showarticletitle{The elliptic curve digital signature algorithm
  (ECDSA)}.
\newblock \bibinfo{journal}{\emph{International journal of information
  security}} \bibinfo{volume}{1}, \bibinfo{number}{1} (\bibinfo{year}{2001}),
  \bibinfo{pages}{36--63}.
\newblock


\bibitem[Kingma and Ba(2014)]%
        {kingma2014adam}
\bibfield{author}{\bibinfo{person}{Diederik~P Kingma} {and}
  \bibinfo{person}{Jimmy Ba}.} \bibinfo{year}{2014}\natexlab{}.
\newblock \showarticletitle{Adam: A method for stochastic optimization}.
\newblock \bibinfo{journal}{\emph{arXiv preprint arXiv:1412.6980}}
  (\bibinfo{year}{2014}).
\newblock


\bibitem[Kirchner and Johnson(2019)]%
        {kirchner2019spn}
\bibfield{author}{\bibinfo{person}{Matthias Kirchner} {and}
  \bibinfo{person}{Cameron Johnson}.} \bibinfo{year}{2019}\natexlab{}.
\newblock \showarticletitle{SPN-CNN: boosting sensor-based source camera
  attribution with deep learning}. In \bibinfo{booktitle}{\emph{2019 IEEE
  International Workshop on Information Forensics and Security (WIFS)}}. IEEE,
  \bibinfo{pages}{1--6}.
\newblock


\bibitem[Konig et~al\mbox{.}(2009)]%
        {konig2009operational}
\bibfield{author}{\bibinfo{person}{Robert Konig}, \bibinfo{person}{Renato
  Renner}, {and} \bibinfo{person}{Christian Schaffner}.}
  \bibinfo{year}{2009}\natexlab{}.
\newblock \showarticletitle{The operational meaning of min-and max-entropy}.
\newblock \bibinfo{journal}{\emph{IEEE Transactions on Information theory}}
  \bibinfo{volume}{55}, \bibinfo{number}{9} (\bibinfo{year}{2009}),
  \bibinfo{pages}{4337--4347}.
\newblock


\bibitem[Kordopatis-Zilos et~al\mbox{.}(2017)]%
        {kordopatis2017near}
\bibfield{author}{\bibinfo{person}{Giorgos Kordopatis-Zilos},
  \bibinfo{person}{Symeon Papadopoulos}, \bibinfo{person}{Ioannis Patras},
  {and} \bibinfo{person}{Yiannis Kompatsiaris}.}
  \bibinfo{year}{2017}\natexlab{}.
\newblock \showarticletitle{Near-duplicate video retrieval with deep metric
  learning}. In \bibinfo{booktitle}{\emph{Proceedings of the IEEE international
  conference on computer vision workshops}}. \bibinfo{pages}{347--356}.
\newblock


\bibitem[Lawgaly and Khelifi(2016)]%
        {lawgaly2016sensor}
\bibfield{author}{\bibinfo{person}{Ashref Lawgaly} {and} \bibinfo{person}{Fouad
  Khelifi}.} \bibinfo{year}{2016}\natexlab{}.
\newblock \showarticletitle{Sensor pattern noise estimation based on improved
  locally adaptive DCT filtering and weighted averaging for source camera
  identification and verification}.
\newblock \bibinfo{journal}{\emph{IEEE Transactions on Information Forensics
  and Security}} \bibinfo{volume}{12}, \bibinfo{number}{2}
  (\bibinfo{year}{2016}), \bibinfo{pages}{392--404}.
\newblock


\bibitem[Li et~al\mbox{.}(2016)]%
        {li2016identification}
\bibfield{author}{\bibinfo{person}{Haodong Li}, \bibinfo{person}{Weiqi Luo},
  \bibinfo{person}{Xiaoqing Qiu}, {and} \bibinfo{person}{Jiwu Huang}.}
  \bibinfo{year}{2016}\natexlab{}.
\newblock \showarticletitle{Identification of various image operations using
  residual-based features}.
\newblock \bibinfo{journal}{\emph{IEEE Transactions on Circuits and Systems for
  Video Technology}} \bibinfo{volume}{28}, \bibinfo{number}{1}
  (\bibinfo{year}{2016}), \bibinfo{pages}{31--45}.
\newblock


\bibitem[Liu et~al\mbox{.}(2018)]%
        {DBLP:journals/amm/LiuZHL18}
\bibfield{author}{\bibinfo{person}{Haiqing Liu}, \bibinfo{person}{Shiqiang
  Zheng}, \bibinfo{person}{Shuhua Hao}, {and} \bibinfo{person}{Yuancheng Li}.}
  \bibinfo{year}{2018}\natexlab{}.
\newblock \showarticletitle{Multifeature Fusion Detection Method for Fake Face
  Attack in Identity Authentication}.
\newblock \bibinfo{journal}{\emph{Adv. Multim.}}  \bibinfo{volume}{2018}
  (\bibinfo{year}{2018}), \bibinfo{pages}{9025458:1--9025458:10}.
\newblock
\urldef\tempurl%
\url{https://doi.org/10.1155/2018/9025458}
\showDOI{\tempurl}


\bibitem[Lobo et~al\mbox{.}(2008)]%
        {lobo2008auc}
\bibfield{author}{\bibinfo{person}{Jorge~M Lobo}, \bibinfo{person}{Alberto
  Jim{\'e}nez-Valverde}, {and} \bibinfo{person}{Raimundo Real}.}
  \bibinfo{year}{2008}\natexlab{}.
\newblock \showarticletitle{AUC: a misleading measure of the performance of
  predictive distribution models}.
\newblock \bibinfo{journal}{\emph{Global ecology and Biogeography}}
  \bibinfo{volume}{17}, \bibinfo{number}{2} (\bibinfo{year}{2008}),
  \bibinfo{pages}{145--151}.
\newblock


\bibitem[Lukas et~al\mbox{.}(2006)]%
        {lukas2006digital}
\bibfield{author}{\bibinfo{person}{Jan Lukas}, \bibinfo{person}{Jessica
  Fridrich}, {and} \bibinfo{person}{Miroslav Goljan}.}
  \bibinfo{year}{2006}\natexlab{}.
\newblock \showarticletitle{Digital camera identification from sensor pattern
  noise}.
\newblock \bibinfo{journal}{\emph{IEEE Transactions on Information Forensics
  and Security}} \bibinfo{volume}{1}, \bibinfo{number}{2}
  (\bibinfo{year}{2006}), \bibinfo{pages}{205--214}.
\newblock


\bibitem[Lyu et~al\mbox{.}(2014)]%
        {lyu2014exposing}
\bibfield{author}{\bibinfo{person}{Siwei Lyu}, \bibinfo{person}{Xunyu Pan},
  {and} \bibinfo{person}{Xing Zhang}.} \bibinfo{year}{2014}\natexlab{}.
\newblock \showarticletitle{Exposing region splicing forgeries with blind local
  noise estimation}.
\newblock \bibinfo{journal}{\emph{International journal of computer vision}}
  \bibinfo{volume}{110}, \bibinfo{number}{2} (\bibinfo{year}{2014}),
  \bibinfo{pages}{202--221}.
\newblock


\bibitem[Martinkauppi et~al\mbox{.}(2020)]%
        {martinkauppi2020design}
\bibfield{author}{\bibinfo{person}{Louise~Bergman Martinkauppi},
  \bibinfo{person}{Qiuping He}, {and} \bibinfo{person}{Dragos Ilie}.}
  \bibinfo{year}{2020}\natexlab{}.
\newblock \showarticletitle{On the design and performance of Chinese
  OSCCA-approved cryptographic algorithms}. In \bibinfo{booktitle}{\emph{2020
  13th International Conference on Communications (COMM)}}. IEEE,
  \bibinfo{pages}{119--124}.
\newblock


\bibitem[Mouton et~al\mbox{.}(2021)]%
        {mouton2021stride}
\bibfield{author}{\bibinfo{person}{Coenraad Mouton},
  \bibinfo{person}{Johannes~C Myburgh}, {and} \bibinfo{person}{Marelie~H
  Davel}.} \bibinfo{year}{2021}\natexlab{}.
\newblock \showarticletitle{Stride and translation invariance in CNNs}. In
  \bibinfo{booktitle}{\emph{Southern African Conference for Artificial
  Intelligence Research}}. Springer, \bibinfo{pages}{267--281}.
\newblock


\bibitem[Quan(2020)]%
        {DBLP:phd/ethos/Quan20}
\bibfield{author}{\bibinfo{person}{Yijun Quan}.}
  \bibinfo{year}{2020}\natexlab{}.
\newblock \emph{\bibinfo{title}{Photo response non-uniformity based image
  forensics in the presence of challenging factors}}.
\newblock \bibinfo{thesistype}{Ph.\,D. Dissertation}.
  \bibinfo{school}{University of Warwick, Coventry, {UK}}.
\newblock
\urldef\tempurl%
\url{https://ethos.bl.uk/OrderDetails.do?uin=uk.bl.ethos.837398}
\showURL{%
\tempurl}


\bibitem[Quiring et~al\mbox{.}(2019)]%
        {quiring2019security}
\bibfield{author}{\bibinfo{person}{Erwin Quiring}, \bibinfo{person}{Matthias
  Kirchner}, {and} \bibinfo{person}{Konrad Rieck}.}
  \bibinfo{year}{2019}\natexlab{}.
\newblock \showarticletitle{On the security and applicability of fragile camera
  fingerprints}. In \bibinfo{booktitle}{\emph{European Symposium on Research in
  Computer Security}}. Springer, \bibinfo{pages}{450--470}.
\newblock


\bibitem[Rachmawati et~al\mbox{.}(2018)]%
        {rachmawati2018comparative}
\bibfield{author}{\bibinfo{person}{Dian Rachmawati}, \bibinfo{person}{JT
  Tarigan}, {and} \bibinfo{person}{ABC Ginting}.}
  \bibinfo{year}{2018}\natexlab{}.
\newblock \showarticletitle{A comparative study of Message Digest 5 (MD5) and
  SHA256 algorithm}. In \bibinfo{booktitle}{\emph{Journal of Physics:
  Conference Series}}, Vol.~\bibinfo{volume}{978}. IOP Publishing,
  \bibinfo{pages}{012116}.
\newblock


\bibitem[Ridler et~al\mbox{.}(1978)]%
        {ridler1978picture}
\bibfield{author}{\bibinfo{person}{TW Ridler}, \bibinfo{person}{S Calvard},
  {et~al\mbox{.}}} \bibinfo{year}{1978}\natexlab{}.
\newblock \showarticletitle{Picture thresholding using an iterative selection
  method}.
\newblock \bibinfo{journal}{\emph{IEEE trans syst Man Cybern}}
  \bibinfo{volume}{8}, \bibinfo{number}{8} (\bibinfo{year}{1978}),
  \bibinfo{pages}{630--632}.
\newblock


\bibitem[Rigueira et~al\mbox{.}(2022)]%
        {DBLP:journals/corr/abs-2207-01323}
\bibfield{author}{\bibinfo{person}{Xurxo Rigueira}, \bibinfo{person}{Javier
  Martinez}, \bibinfo{person}{Maria Araujo}, {and} \bibinfo{person}{Antonio
  Recaman}.} \bibinfo{year}{2022}\natexlab{}.
\newblock \showarticletitle{Computer vision application for improved product
  traceability in the granite manufacturing industry}.
\newblock \bibinfo{journal}{\emph{CoRR}}  \bibinfo{volume}{abs/2207.01323}
  (\bibinfo{year}{2022}).
\newblock
\urldef\tempurl%
\url{https://doi.org/10.48550/arXiv.2207.01323}
\showDOI{\tempurl}
\showeprint[arXiv]{2207.01323}


\bibitem[Rogaway and Shrimpton(2004)]%
        {rogaway2004cryptographic}
\bibfield{author}{\bibinfo{person}{Phillip Rogaway} {and}
  \bibinfo{person}{Thomas Shrimpton}.} \bibinfo{year}{2004}\natexlab{}.
\newblock \showarticletitle{Cryptographic hash-function basics: Definitions,
  implications, and separations for preimage resistance, second-preimage
  resistance, and collision resistance}. In
  \bibinfo{booktitle}{\emph{International workshop on fast software
  encryption}}. Springer, \bibinfo{pages}{371--388}.
\newblock


\bibitem[Seger(2012)]%
        {seger2012generalized}
\bibfield{author}{\bibinfo{person}{Olle Seger}.}
  \bibinfo{year}{2012}\natexlab{}.
\newblock \showarticletitle{Generalized and separable sobel operators}.
\newblock \bibinfo{journal}{\emph{Machine vision for three-dimensional scenes}}
  (\bibinfo{year}{2012}), \bibinfo{pages}{347}.
\newblock


\bibitem[Shi et~al\mbox{.}(2016)]%
        {shi2016real}
\bibfield{author}{\bibinfo{person}{Wenzhe Shi}, \bibinfo{person}{Jose
  Caballero}, \bibinfo{person}{Ferenc Husz{\'a}r}, \bibinfo{person}{Johannes
  Totz}, \bibinfo{person}{Andrew~P Aitken}, \bibinfo{person}{Rob Bishop},
  \bibinfo{person}{Daniel Rueckert}, {and} \bibinfo{person}{Zehan Wang}.}
  \bibinfo{year}{2016}\natexlab{}.
\newblock \showarticletitle{Real-time single image and video super-resolution
  using an efficient sub-pixel convolutional neural network}. In
  \bibinfo{booktitle}{\emph{Proceedings of the IEEE conference on computer
  vision and pattern recognition}}. \bibinfo{pages}{1874--1883}.
\newblock


\bibitem[Shullani et~al\mbox{.}(2017)]%
        {DBLP:journals/ejisec/ShullaniFISP17}
\bibfield{author}{\bibinfo{person}{Dasara Shullani}, \bibinfo{person}{Marco
  Fontani}, \bibinfo{person}{Massimo Iuliani}, \bibinfo{person}{Omar~Al Shaya},
  {and} \bibinfo{person}{Alessandro Piva}.} \bibinfo{year}{2017}\natexlab{}.
\newblock \showarticletitle{{VISION:} a video and image dataset for source
  identification}.
\newblock \bibinfo{journal}{\emph{{EURASIP} J. Inf. Secur.}}
  \bibinfo{volume}{2017} (\bibinfo{year}{2017}), \bibinfo{pages}{15}.
\newblock
\urldef\tempurl%
\url{https://doi.org/10.1186/s13635-017-0067-2}
\showDOI{\tempurl}


\bibitem[Stewart(2012)]%
        {DBLP:journals/jthtl/Stewart12}
\bibfield{author}{\bibinfo{person}{Daxton~R. Stewart}.}
  \bibinfo{year}{2012}\natexlab{}.
\newblock \showarticletitle{Can {I} Use this Photo {I} Found on Facebook?
  Applying Copyright Law and Fair Use Analysis to Photographs on Social
  Networking Sites Republished for News Reporting Purposes}.
\newblock \bibinfo{journal}{\emph{J. Telecommun. High Technol. Law}}
  \bibinfo{volume}{10}, \bibinfo{number}{1} (\bibinfo{year}{2012}),
  \bibinfo{pages}{93--122}.
\newblock
\urldef\tempurl%
\url{http://www.jthtl.org/content/articles/V10I1/JTHTLv10i1\_Stewart.PDF}
\showURL{%
\tempurl}


\bibitem[Valsesia et~al\mbox{.}(2015)]%
        {valsesia2015compressed}
\bibfield{author}{\bibinfo{person}{Diego Valsesia}, \bibinfo{person}{Giulio
  Coluccia}, \bibinfo{person}{Tiziano Bianchi}, {and} \bibinfo{person}{Enrico
  Magli}.} \bibinfo{year}{2015}\natexlab{}.
\newblock \showarticletitle{Compressed fingerprint matching and camera
  identification via random projections}.
\newblock \bibinfo{journal}{\emph{IEEE Transactions on Information Forensics
  and Security}} \bibinfo{volume}{10}, \bibinfo{number}{7}
  (\bibinfo{year}{2015}), \bibinfo{pages}{1472--1485}.
\newblock


\bibitem[Valsesia et~al\mbox{.}(2017)]%
        {valsesia2017user}
\bibfield{author}{\bibinfo{person}{Diego Valsesia}, \bibinfo{person}{Giulio
  Coluccia}, \bibinfo{person}{Tiziano Bianchi}, {and} \bibinfo{person}{Enrico
  Magli}.} \bibinfo{year}{2017}\natexlab{}.
\newblock \showarticletitle{User authentication via PRNU-based physical
  unclonable functions}.
\newblock \bibinfo{journal}{\emph{IEEE Transactions on Information Forensics
  and Security}} \bibinfo{volume}{12}, \bibinfo{number}{8}
  (\bibinfo{year}{2017}), \bibinfo{pages}{1941--1956}.
\newblock


\bibitem[Wang et~al\mbox{.}(2022)]%
        {DBLP:journals/tnse/WangSWZ22}
\bibfield{author}{\bibinfo{person}{Baowei Wang}, \bibinfo{person}{Jiawei Shi},
  \bibinfo{person}{Weishen Wang}, {and} \bibinfo{person}{Peng Zhao}.}
  \bibinfo{year}{2022}\natexlab{}.
\newblock \showarticletitle{Image Copyright Protection Based on Blockchain and
  Zero-Watermark}.
\newblock \bibinfo{journal}{\emph{{IEEE} Trans. Netw. Sci. Eng.}}
  \bibinfo{volume}{9}, \bibinfo{number}{4} (\bibinfo{year}{2022}),
  \bibinfo{pages}{2188--2199}.
\newblock
\urldef\tempurl%
\url{https://doi.org/10.1109/TNSE.2022.3157867}
\showDOI{\tempurl}


\bibitem[Wang et~al\mbox{.}(2014)]%
        {wang2014learning}
\bibfield{author}{\bibinfo{person}{Jiang Wang}, \bibinfo{person}{Yang Song},
  \bibinfo{person}{Thomas Leung}, \bibinfo{person}{Chuck Rosenberg},
  \bibinfo{person}{Jingbin Wang}, \bibinfo{person}{James Philbin},
  \bibinfo{person}{Bo Chen}, {and} \bibinfo{person}{Ying Wu}.}
  \bibinfo{year}{2014}\natexlab{}.
\newblock \showarticletitle{Learning fine-grained image similarity with deep
  ranking}. In \bibinfo{booktitle}{\emph{Proceedings of the IEEE conference on
  computer vision and pattern recognition}}. \bibinfo{pages}{1386--1393}.
\newblock


\bibitem[Wei et~al\mbox{.}(2020)]%
        {wei2020physics}
\bibfield{author}{\bibinfo{person}{Kaixuan Wei}, \bibinfo{person}{Ying Fu},
  \bibinfo{person}{Jiaolong Yang}, {and} \bibinfo{person}{Hua Huang}.}
  \bibinfo{year}{2020}\natexlab{}.
\newblock \showarticletitle{A physics-based noise formation model for extreme
  low-light raw denoising}. In \bibinfo{booktitle}{\emph{Proceedings of the
  IEEE/CVF Conference on Computer Vision and Pattern Recognition}}.
  \bibinfo{pages}{2758--2767}.
\newblock


\bibitem[Xu et~al\mbox{.}(2016)]%
        {DBLP:conf/uss/XuPFM16}
\bibfield{author}{\bibinfo{person}{Yi Xu}, \bibinfo{person}{True Price},
  \bibinfo{person}{Jan{-}Michael Frahm}, {and} \bibinfo{person}{Fabian
  Monrose}.} \bibinfo{year}{2016}\natexlab{}.
\newblock \showarticletitle{Virtual {U:} Defeating Face Liveness Detection by
  Building Virtual Models from Your Public Photos}. In
  \bibinfo{booktitle}{\emph{25th {USENIX} Security Symposium, {USENIX} Security
  16, Austin, TX, USA, August 10-12, 2016}},
  \bibfield{editor}{\bibinfo{person}{Thorsten Holz} {and}
  \bibinfo{person}{Stefan Savage}} (Eds.). \bibinfo{publisher}{{USENIX}
  Association}, \bibinfo{pages}{497--512}.
\newblock
\urldef\tempurl%
\url{https://www.usenix.org/conference/usenixsecurity16/technical-sessions/presentation/xu}
\showURL{%
\tempurl}


\bibitem[Yao and Zhang(2022)]%
        {DBLP:journals/sensors/YaoZ22}
\bibfield{author}{\bibinfo{person}{Qi Yao} {and} \bibinfo{person}{Huajun
  Zhang}.} \bibinfo{year}{2022}\natexlab{}.
\newblock \showarticletitle{Improving Agricultural Product Traceability Using
  Blockchain}.
\newblock \bibinfo{journal}{\emph{Sensors}} \bibinfo{volume}{22},
  \bibinfo{number}{9} (\bibinfo{year}{2022}), \bibinfo{pages}{3388}.
\newblock
\urldef\tempurl%
\url{https://doi.org/10.3390/s22093388}
\showDOI{\tempurl}


\bibitem[Yoshimura et~al\mbox{.}(2022)]%
        {Yoshimura2022DynamicISPDC}
\bibfield{author}{\bibinfo{person}{Masakazu Yoshimura}, \bibinfo{person}{Junji
  Otsuka}, \bibinfo{person}{Atsushi Irie}, {and} \bibinfo{person}{Takeshi
  Ohashi}.} \bibinfo{year}{2022}\natexlab{}.
\newblock \showarticletitle{DynamicISP: Dynamically Controlled Image Signal
  Processor for Image Recognition}.
\newblock \bibinfo{journal}{\emph{ArXiv}}  \bibinfo{volume}{abs/2211.01146}
  (\bibinfo{year}{2022}).
\newblock


\bibitem[Zhang et~al\mbox{.}(2017)]%
        {zhang2017beyond}
\bibfield{author}{\bibinfo{person}{Kai Zhang}, \bibinfo{person}{Wangmeng Zuo},
  \bibinfo{person}{Yunjin Chen}, \bibinfo{person}{Deyu Meng}, {and}
  \bibinfo{person}{Lei Zhang}.} \bibinfo{year}{2017}\natexlab{}.
\newblock \showarticletitle{Beyond a gaussian denoiser: Residual learning of
  deep cnn for image denoising}.
\newblock \bibinfo{journal}{\emph{IEEE transactions on image processing}}
  \bibinfo{volume}{26}, \bibinfo{number}{7} (\bibinfo{year}{2017}),
  \bibinfo{pages}{3142--3155}.
\newblock


\end{thebibliography}

\appendix

\section{Security Evaluation of Spatial Splitting}

For evaluating the security of spatial splitting, we first derive one fingerprint from each RAW odd photo and one fingerprint from each RAW even photo for our benchmark test set (e.g., 15 iPhone cameras) based on our trained network, resulting in two sets of 1,665 fingerprints. Then for each RAW photo, we calculate NCC (Normalized Cross-correlation Coefficient) from its corresponding RAW odd fingerprint and RAW even fingerprint. Finally, for each camera, we calculate AUC from its NCC of RAW photos over two parts (odd and even). Figure 9 illustrates the NCC over two parts (odd and even) and AUC for each camera. We got an average of 96.22\% AUC with 5.33\% standard deviation, which indicates relatively low information leakage.

\begin{figure}[htb]
\centering
\includegraphics[width=\linewidth]{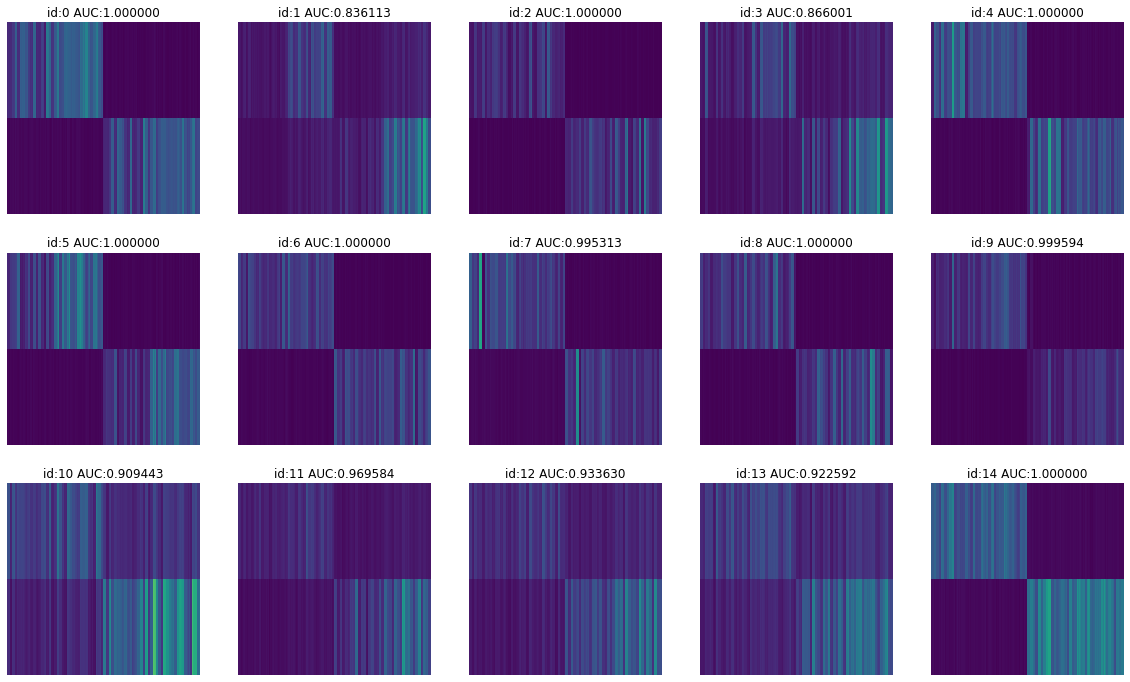} 
\caption{Normalized Cross-correlation Coefficient over two parts (odd and even) and AUC for each iPhone camera of benchmark test set}
\label{fig11}
\end{figure}

Furthermore, we calculated AUC and EER from the correlation matrix between RAW odd fingerprints and the correlation matrix between RAW even fingerprints, i.e., two 1,665 $\times$ 1,665 matrices . The results show 99.99\% AUC and 0.253\% EER for RAW odd fingerprints, and 99.92\% AUC and 0.497\% EER for RAW even fingerprints, both indicating highly discriminative ability.

\section{Network Performance on Android RAW photos and JPEG photos}

While our network was trained only on iPhone RAW photos, it displayed superior generalization and adaptability on both RAW Android photos and JPEG compressed photos.

For examining Android RAW photos, we provide an additional test dataset with 1,276 RAW photos from 15 Android smartphone cameras. Table 3 shows the fingerprint accuracy performance comparison of our algorithm with previous algorithms on this dataset. As shown in the table, our model outperforms conventional algorithms by a large margin with much higher AUC and lower EER.

\begin{table}[h!]
\caption{Fingerprint accuracy performance comparison of ours with previously open-sourced fingerprint extraction algorithms on Android RAW photos. Result with $^*$ indicates containing post-processing (ZM \& WF).}
\centering
\label{table3}
\scalebox{1}{
\begin{tabular}{c|l||c|c}
\toprule
{Register} & {Method}   &  AUC$^*$ $\uparrow$ & EER$^*$ $\downarrow$ \\ \toprule
\multirow{3}{*}{\begin{tabular}[c]{@{}c@{}}Single\end{tabular}}    & PRNU        &  99.81\%  & 1.600\%\\
                                                                             & Noiseprint        &  55.49\%  & 45.53\%   \\  
                                                                             & Ours    & \textbf{99.94\%}  & \textbf{0.907\%} \\ 
                                                                            
                                        \midrule                                 
\multirow{3}{*}{\begin{tabular}[c]{@{}c@{}}Multiple\end{tabular}} & PRNU        & 99.99\%  & 0.179\% \\ 
                                                                             & Noiseprint  & 51.39\%  & 49.23\%   \\  
                                                                             & Ours  & \textbf{100.0\%}  & \textbf{0.0\%}  \\  
                                                                             
                                                                         \bottomrule
\end{tabular}}
\end{table}

For examining JPEG compressed photos, we directly tested our released model on VISION dataset \cite{DBLP:journals/ejisec/ShullaniFISP17} (35 devices with 34,427 JPEG photos). On this JPEG compressed dataset we obtained 92.83\% AUC, indicating better discrimination than other SOTA methods \cite{lukas2006digital}.

\section{Security analysis in Detail}

Here we give detailed proofs to the theorems presented in security analysis subsection 5.3.

\begin{figure}[htb!]\label{fig:9}
\includegraphics[width=\linewidth]{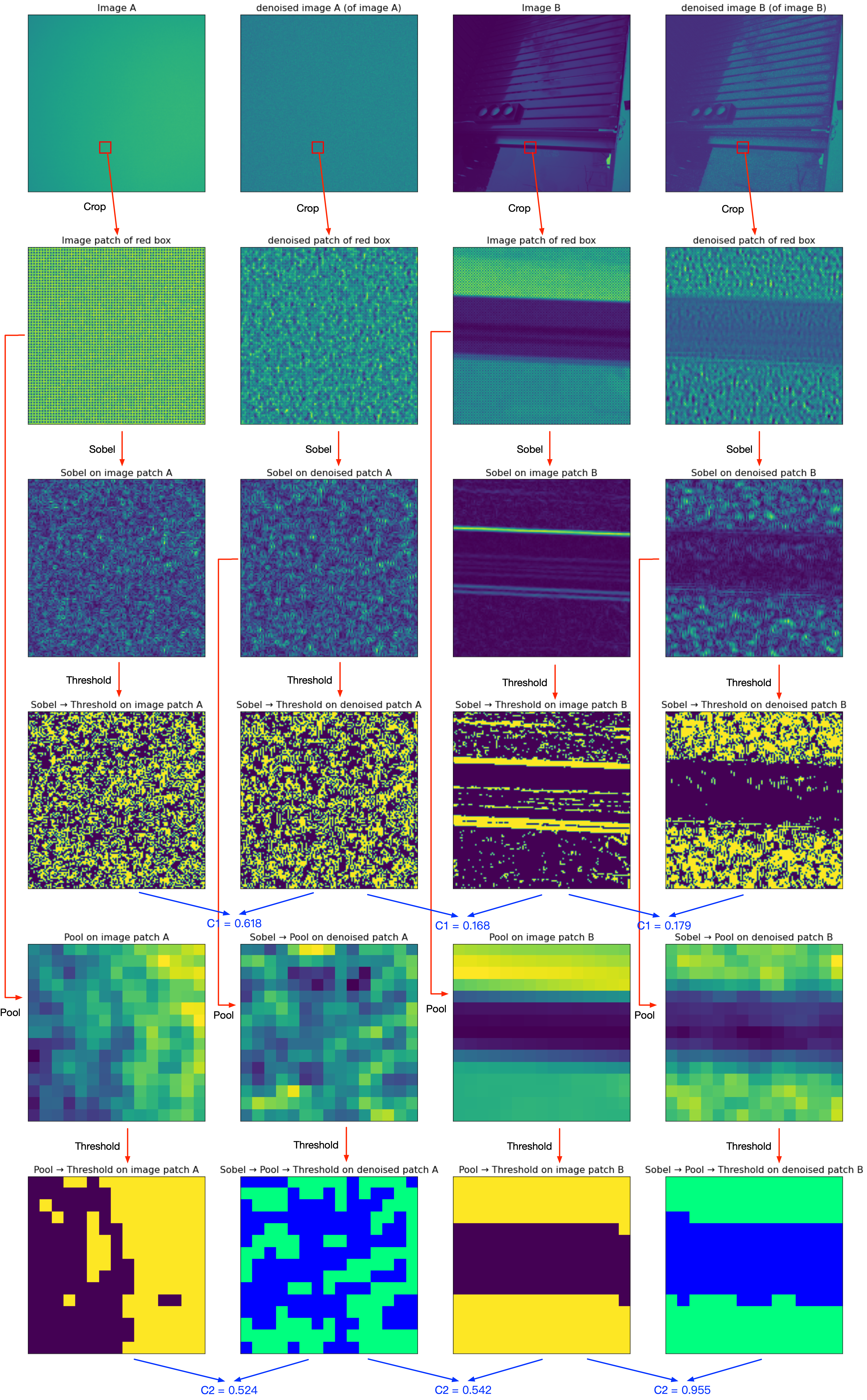}
\caption{Visualization of consistency coefficients.}
\end{figure}

\begin{figure}[htb!]\label{fig:10}
\includegraphics[width=\linewidth]{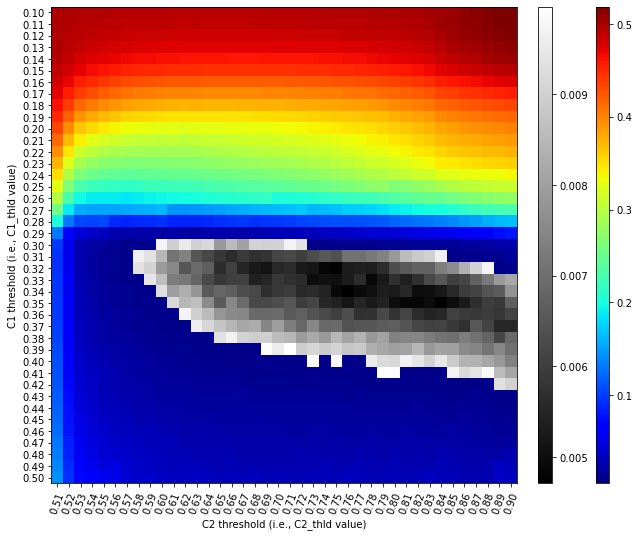}
\caption{Grid search on values of $C1\_thld$ and $C2\_thld$ against EER of patch level consistency checking results.}
\end{figure}

\begin{theorem}
If an attacker do not have access to the source camera, RAW odd photo $O$, fingerprint $K$ and private key $sk$, then the probability for the attacker to successfully forge a cryptographic secure signature (e.g., ECDSA, SM2 etc.,) with public key $pk$ is $P_a \leq \frac{1}{2^{\lambda-1}}$ where $\lambda$ is the security parameter.
\end{theorem}

\begin{proof}
Min-entropy \cite{konig2009operational} describes the uncertainty of a random value. As studied by Bayram et al. \cite{bayram2012efficient}, the signed fingerprints extracted from photos are truly random, hence every single bit of the fingerprint is independent from all the other fingerprint bits. Thus, neither the fingerprint extracted from RAW even photo nor the fingerprints extracted from other photos are helpless for an attacker to exploit the target odd signed fingerprint:
\begin{equation}\begin{aligned}
H_\infty (K_{raw1,odd}) = \tilde{H}_\infty (K_{raw1,odd}|K_{raw1,even})
\\ = \tilde{H}_\infty (K_{raw1,odd}|K_{raw*,\cdot })=m 
\end{aligned}\end{equation}
where $\tilde{H}_\infty (A|B)$ denotes the average min-entropy of A given B, $K_{raw*,\cdot}$ denotes the $\cdot$ fingerprint extracted from (odd or even) photo taken by cameras $*$.

Recall that we hide the user’s secret key $sk$ with secure sketch $s = K \oplus \mathbb{C}(sk)$. With $K$ truly random, an attacker can do nothing better than randomly generate a new fingerprint $K'$ and try to decode $sk' = \mathbb{D}(K'\oplus s)$ and see if $sk = sk'$. According to the work of Valsesia et al. \cite{valsesia2017user}, the probability for an attacker to recover user's secret key $sk$ is upper bounded by: 
\begin{equation}\begin{aligned}
P_{adv} = \mathbb{E}_{sk}[\mathbb{P}(K\in C_{sk})] = \frac{1}{2^{m+\lambda}}\sum |C_{sk}|\le \frac{1}{2^\lambda}
\end{aligned}\end{equation}

Besides recovering secret key from secure sketch, the attacker can also try to recover from user public key, however, as long as the user use cryptographic safe signature schemes such as ECDSA, SM2, BLS etc., and complies with key generation rules, then according to the Elliptic Curve Discrete Log Problem (ECDLP) \cite{galbraith2016recent}, the attacker can recover secret key $sk$ from public key $pk = g^{sk}, g\in \mathbb{G}$ and $\mathbb{G}$ is the group of elliptic curve points, with negligible probability $negl(\lambda)\le 2^{128}$,for instance, let $sk \in \mathbb{F}_q$, ${F}_q$ refers to the finite field modulo prime $q$, $|q| = 256$ and $\lambda=128$, breaking $sk$ from $pk$ requires averagely $2^{128}$ operations, thus achieve 128-bit security.

In conclusion, the attacker has less than $\frac{1}{2^{\lambda-1}}$ probability to successfully forge a signature. 
\end{proof}

\begin{theorem}
Let $hash(\cdot)$ be a cryptographic secure hash function (e.g., SHA256, SHA3 etc.), if the attacker do not have access to fingerprint $K$ and can not break the pre-image resistance property of $hash(\cdot)$ \cite{rogaway2004cryptographic}, then the attacker can forge a prove of statement \eqref{eqn:zkpstatement} with probability $P_b \leq \frac{1}{2^m}+\frac{1}{2^{2\cdot \lambda}}$ where $m$ is the bit length of fingerprint $K$,  $\lambda$ is the security parameter and a pre-image here refers to the message mapped to a particular digest via hashing.
\end{theorem}

\begin{proof}
In our solution, if the attacker is able to forge a set of inputs that complies the statement and pass the verification. He must be able to either regenerate fingerprint protected by user, or find another fingerprint the digest $h'$ of which is identical to the registered digest $h$. Since operations such as subtraction and $Sobel$ operations are reversible, if an attacker is capable to regenerate the fingerprint $K$, he could reversely comply with the statement to generate a fake RAW photo. However, according to our previous discussion of entropy of sign fingerprint, an attacker with no access to user’s camera is only capable to reconstruct the random identical fingerprint $K'=K$ with negligibly $\frac{1}{2^m}$. 

The attacker could also try to find another fingerprint producing the same digest value than the registered one.
However, according to pre-image resistance property of cryptographic secure hash function, an attacker can do nothing but the brute-force attack to find the pre-image from digest. For instance, probability to successfully brute force pre-image of SHA256 is $\frac{1}{2^{2\cdot \lambda}} = \frac{1}{2^{256}}$.

Thus, if the attacker is unable to stole user’s fingerprint $K$, he has less than $\frac{1}{2^m}+\frac{1}{2^{2\cdot \lambda}}$ probability to forge a our ZKP statement. 
\end{proof}

\section{Consistency Checking in Detail}

\subsection{Consistency Coefficient Visualization}

To illustrate the ability of two consistency coefficients $C1$ and $C2$, we visualize an example in Figure 9. As show in the figure, $C1$ focuses on close-up consistency while $C2$ focuses on contour consistency. Collaboratively using $C1$ and $C2$ can detect almost all the near duplicate patches (i.e., image patches and their denoised version).

\subsection{Feasible Hyperparameter Range}

We partitioned our photo data set into patches to investigate the feasible hyperparameter range. For each patch we use its denoised version as positive samples and randomly choose denoised patches with the same location from three other images as negative samples. In our experiment we used 150,000 positive samples and 450,000 negative samples. Then we grid searched on the values of $C1\_thld$ and $C2\_thld$ against the Equal Error Rate (EER) of patch level consistency checking results.

As show in Figure 10, the black color indicates very low EER which takes a wide range of area. This illustrates that we have a wide range to select feasible combination of $C1\_thld$ and $C2\_thld$ which grantees the reliability of consistency checking in practice.

\end{document}